\documentclass{article}

    \PassOptionsToPackage{numbers, compress}{natbib}

\usepackage[final]{neurips_2024}

\usepackage[utf8]{inputenc} 
\usepackage[T1]{fontenc}    
\usepackage{hyperref}       
\usepackage{url}            
\usepackage{booktabs}       
\usepackage{amsfonts}       
\usepackage{nicefrac}       
\usepackage{microtype}      
\usepackage{xcolor}         
\usepackage{graphicx}
\usepackage{amsmath, amssymb, amsthm}
\usepackage{mathtools}
\usepackage{multirow}
\usepackage{dsfont}
\theoremstyle{plain}
\newtheorem{theorem}{Theorem}[section]
\newtheorem{proposition}[theorem]{Proposition}
\usepackage{subcaption}

\theoremstyle{remark}

\title{2D-OOB: Attributing Data Contribution Through Joint Valuation Framework}

\author{%
 Yifan Sun\thanks{Equal contribution.}\\
  University of Illinois Urbana-Champaign\\
  \texttt{yifan50@illinois.edu} \\
  \And
  Jingyan Shen$^{*}$ \\
  Columbia University\\
  \texttt{js5544@columbia.edu} \\
  \AND
  Yongchan Kwon\thanks{Corresponding author.} \\
  Columbia University\\
  \texttt{yk3012@columbia.edu} 
}

\begin{document}

\maketitle

\begin{abstract}
Data valuation has emerged as a powerful framework for quantifying each datum's contribution to the training of a machine learning model. However, it is crucial to recognize that the quality of \textit{cells} within a single data point can vary greatly in practice. For example, even in the case of an abnormal data point, not all cells are necessarily noisy. The single scalar score assigned by existing data valuation methods blurs the distinction between noisy and clean cells of a data point, making it challenging to interpret the data values. In this paper, we propose \texttt{2D-OOB}, an out-of-bag estimation framework for jointly determining helpful (or detrimental) samples as well as the particular cells that drive them. Our comprehensive experiments demonstrate that \texttt{2D-OOB} achieves state-of-the-art performance across multiple use cases while being exponentially faster. Specifically, \texttt{2D-OOB} shows promising results in detecting and rectifying fine-grained outliers at the cell level, and localizing backdoor triggers in data poisoning attacks. 
\end{abstract}

\section{Introduction}
\label{sec:intro}
From customer behavior prediction and medical image analysis to autonomous driving and policy making, machine learning (ML) systems process ever increasing amounts of data. In such data-rich regimes, a fraction of the samples is often noisy, incorrect annotations are likely to occur, and uniform data quality standards become difficult to enforce. To address these challenges, data valuation emerges as a research field receiving increasing attention, focusing on properly assessing the contribution of each datum to ML training \citep{ghorbani2019data, kwon2023data}. These methods have proven useful in identifying low-quality samples that can be detrimental to model performance, as well as selecting subsets of data that are representative of enhanced model performance \citep{yoon2020data, liang2022advances, wang2024rethinking}. Furthermore, they are widely applicable in data marketplace for fair revenue allocation and incentive design \citep{zhao2023addressing,wang2024efficient, sim2023incentives}.

Nevertheless, existing data valuation methods assign a scalar score to each datum, thereby failing to account for the varied roles of individual cells. This leaves the valuation rationale unclear and can be unsatisfactory and sub-optimal in various practical scenarios. Firstly, whenever a score is assigned to a data point by a particular data valuation method, it is crucial to understand the underlying justifications to ensure transparency and reliability, especially in high-stakes decision making \citep{sim2022data}.
Secondly, it is important to recognize the fact that even if a data point is of low quality, it is rarely the case that all the cells within this data point are noisy \citep{rousseeuw2018detecting, leung2016robust, su2023robust}. The absence of detailed insights into how individual cells contribute to ML training inevitably leads to discarding the entire data point. This can result in substantial data waste, particularly when only a few cells are noisy and data acquisition is expensive.
Finally, in data markets, different cells within a data point may originate from different data sellers \citep{bleiholder2009data, fernandez2020data}. 
Consequently, a singular valuation for the entire point fails to offer equitable compensation to all contributing parties.

\begin{figure}
\centering
  \includegraphics[width = \textwidth]{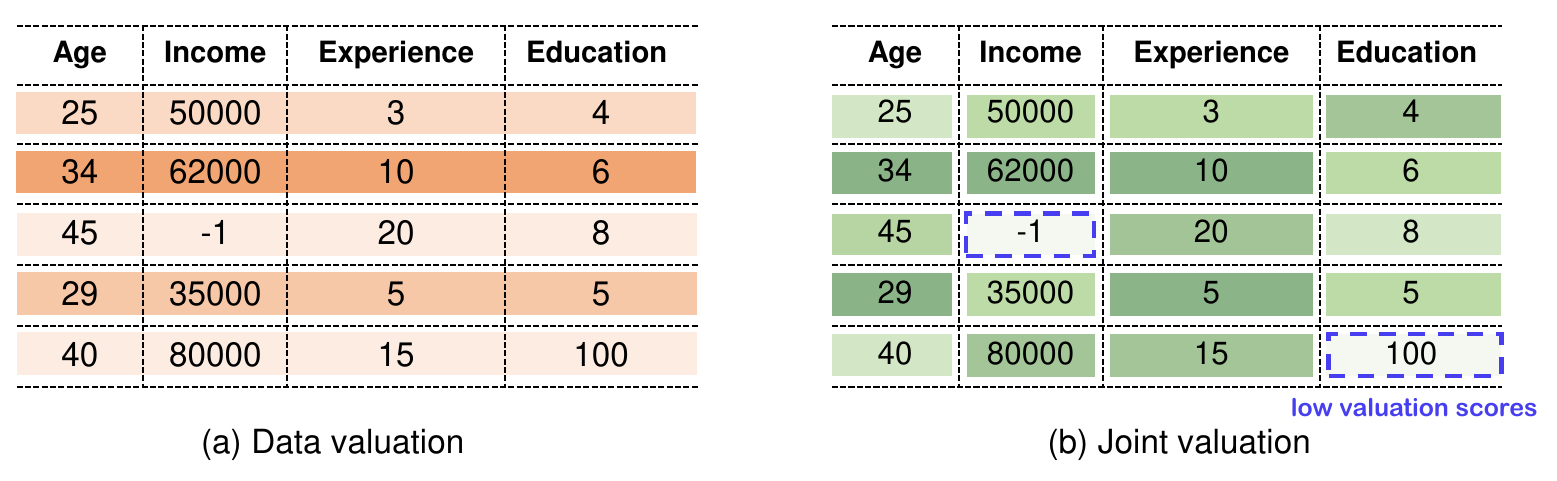}
  \caption{\textbf{Comparison of data valuation and joint valuation.} (a) Data valuation evaluates the quality of individual data points, whereas (b) joint valuation evaluates the quality of individual cells. Both panels illustrate the same hypothetical dataset, while darker colors indicate higher quality or importance. 
  As illustrated in panel (a), data valuation can only identify that the third and fifth data points are of low quality, but it lacks further feature-level attribution. This limitation may result in discarding the entire data point, even when only certain cells are problematic. In contrast, joint valuation provides a finer level of attribution than data valuation and aims to reveal how individual features contribute to data values. As shown in panel (b), the joint valuation framework can identify outlier cells (highlighted by blue boxes), such as $-1$ in "Income" and $100$ in "Education", providing detailed interpretations of data values.
  } 
\label{fig: comparison}
\end{figure}

\paragraph{Our contributions} In this paper, we propose \texttt{2D-OOB}, a powerful and efficient joint valuation framework that can attribute a data point's value to its features. \texttt{2D-OOB} quantifies the importance of each \textit{cell} in a dataset, as illustrated in Figure \ref{fig: comparison}, providing interpretable insights into which cells are associated with influential data points. Our method is computationally efficient as well as theoretically supported by its connections with \texttt{Data-OOB} \citep{kwon2023data}. 
Moreover, our extensive empirical experiments demonstrate the practical effectiveness of \texttt{2D-OOB} in various use cases. \texttt{2D-OOB} accurately identifies cell outliers and pinpoints which cells to fix to improve model performance. Additionally, \texttt{2D-OOB} enables inspection of data poisoning attacks by precisely localizing the backdoor trigger, an artifact inserted into a training sample to induce malicious model behavior \citep{gu2017badnets,chen2017targeted}. On average, \texttt{2D-OOB} is $200$ times faster than state-of-the-art methods across all datasets examined. 

\section{Preliminaries}

\paragraph{Notations} 
Throughout this paper, we focus on supervised learning settings. For $d \in \mathbb{N}$, we denote an input space and an output space by $\mathcal{X} \subseteq \mathbb{R}^d$ and $\mathcal{Y} \subseteq \mathbb{R}$, respectively. We denote a training dataset with $n$ data points by $\mathcal{D}=\{(x_i, y_i)\}_{i=1} ^n$ where $(x_i, y_i)$ is the $i$-th pair of the input covariates $x_i \in \mathcal{X}$ and its output label $y_i \in \mathcal{Y}$. For an event $A$, an indicator function $\mathds{1}(A)$ is $1$ if $A$ is true, otherwise $0$. For $j \in \mathbb{N}$, we set $[j]:=\{1, \dots, j\}$. For a set $S$, we denote its power set by $2^S$ and its cardinality by $|S|$.  

%

\paragraph{DataShapley}
\label{sec: marginal}
The primary goal of data valuation is to quantify the contribution of individual data points to a model's performance. Leveraging the Shapley value in cooperative game theory \citep{shapley1953value}, \texttt{DataShapley} \citep{ghorbani2019data} measures the average change in a utility function $U:2^\mathcal{D}\rightarrow\mathbb{R} $ over all possible subsets of the dataset that either include or exclude the data point. For $i\in [n]$, \texttt{DataShapley} of $i$-th datum is defined as follows.
\begin{align}
    \phi_i^{\mathrm{Shap}}:= \frac{1}{n}\sum_{k=1}^{n}\frac{1}{\binom{n-1}{k-1}}\sum_{S\in\mathcal{D}_k^{(i)}} [U(S\cup\{(x_i,y_i)\})-U(S)],
    \label{eqn:datashap}
\end{align}
where $\mathcal{D}_k^{(i)}:=\{S\subseteq\mathcal{D} | (x_i,y_i)\notin S,|S|=k-1\}$.
\texttt{DataShapley} $\phi_i^{\mathrm{Shap}}$ in \eqref{eqn:datashap} considers every set $S \in \mathcal{D}_k^{(i)}$ and computes the average difference in utility $U(S\cup\{(x_i,y_i)\})-U(S)$.
It characterizes the impact of a data point, but its computation requires evaluating $U$ for all possible subsets of $\mathcal{D}$, rendering precise calculations infeasible. Many efficient computation algorithms have been proposed \citep{jia2019efficient, pmlr-v130-kwon21a, 10.1145/3588728}, and in these studies, Shapley-based methods have demonstrated better effectiveness in detecting low-quality samples than standard attribution approaches, such as leave-one-out and influence function methods \citep{koh2017understanding, feldman2020neural}.

\paragraph{Data-OOB} As an alternative efficient data valuation method, \citet{kwon2023data} propose \texttt{Data-OOB}, which leverages a bagging model and measures the similarity between the nominal label and weak learners' predictions. To be more specific, consider a bagging model consisting of $B$ weak learners, where for $b\in [B]$, the $b$-th weak learner $\hat{h}_b$ is given as a minimizer of the weighted empirical risk,
\begin{equation*}
    \hat{h}_b := \mathrm{argmin}_{h} \sum_{i=1} ^n w_{bi} \ell(y_{i} , h(x_{i})),
    \label{eqn:weak_learner_bagging}
\end{equation*}
where $\ell:\mathcal{Y} \times \mathcal{Y} \to \mathbb{R}$ is a loss function and $w_{bi} \in \mathbb{N}$ is the number of times the $i$-th datum $(x_i, y_i)$ is selected by the $b$-th bootstrap dataset. Let $\mathbf{w}_b$ be a weight vector $\mathbf{w}_b:=(w_{b1}, \dots, w_{bn})$ for all $b \in [B]$. For $i \in [n] $ and $ \{(\mathbf{w}_b,\hat{h}_b)\}_{b=1} ^B$, \texttt{Data-OOB} of the $i$-th datum is defined as follows.
\begin{align}
    \phi_i^{\mathrm{OOB}} := \frac{\sum_{b=1} ^B \mathds{1}(w_{bi} =0) T(y_i, \hat{h}_b (x_i)) }{\sum_{b=1} ^B \mathds{1}(w_{bi} =0)},
    \label{eqn:original_OOB}
\end{align}
where $T(y_i, \hat{h}_b (x_i))$ is a score function evaluated at $(x_i, y_i)$. We assume that the higher $T$, the better the prediction. In classification settings, a common choice for $T$ is $\mathds{1}(y_i = \hat{h}_b (x_i))$, and in this case, \texttt{Data-OOB} $\phi_i^{\mathrm{OOB}}$ measures the average similarity between a nominal label $y_i$ and weak learners' predictions $\hat{h}_b (x_i)$ when a datum $(x_i, y_i)$ is \textit{not} sampled in a bootstrap dataset. It intuitively captures the quality of a data point. For instance, when $(x_i, y_i)$ is a mislabeled sample or an outlier, the label $y_i$ is likely to differ from $\hat{h}_b (x_i)$, resulting in $\phi_i^{\mathrm{OOB}}$ being close to zero. 

It is noteworthy that \texttt{Data-OOB} in \eqref{eqn:original_OOB} can be computed by training a single bagging model, making it computationally efficient. \citet{kwon2023data} show that \texttt{Data-OOB} can easily scale to millions of data points, but for \texttt{DataShapley} this is often very impractical. In addition, \texttt{Data-OOB} is typically comparable to or even more effective than \texttt{DataShapley} in detecting mislabeled data points and selecting helpful data points \citep{kwon2023data, jiang2023opendataval}.

\section{Attributing Data Contribution through Joint Valuation Framework}

Data valuation quantifies the utility of data points, however, it fails to identify which features contribute to these data values and to what extent. For instance, in anomaly detection tasks, data valuation methods can be deployed to detect anomalous data points but cannot explain why they are considered abnormal, which is generally not desirable in practice. To address this challenge, we introduce a joint valuation framework that assigns \textit{a cell score} to each feature of a data point. Here, a cell score is designed to quantify how each feature affects the value of an individual data point, thereby attributing the data value to specific features. 


To the best of the author's knowledge, \citet{liu20232d} were the first to consider the concept of joint valuation in the literature, proposing \texttt{2D-Shapley} as a means to quantitatively interpret \texttt{DataShapley}. To formalize this, we denote a 2D utility function by $u:[n]\times [d] \to \mathbb{R}$, which takes as input a subset of data points $S\subseteq[n]$ and a subset of features $F \subseteq [d]$, measuring the utility of a fragment of the given dataset consisting of cells $\{(i,j)\}_{i\in S, j\in F}$, where a tuple $(i,j)$ denotes a cell at the $i$-th datum and the $j$-th column. Then, \texttt{2D-Shapley} is defined as
\begin{align}
    \psi_{ij}^{\mathrm{2D-Shap}}:= \frac{1}{nd}\sum_{k=1}^{n}\sum_{l=1}^{d}\frac{1}{\binom{n-1}{k-1}\binom{d-1}{l-1}}\sum_{(S,F)\in\mathcal{D}_{k,l}^{(i,j)}}M_{u}^{i,j}(S,F)
    \label{eqn:2dshap}
\end{align}
where $\mathcal{D}_{k,l}^{(i,j)}:=\{(S,F)| S\subseteq[n] \backslash \{i\},F\subseteq[d] \backslash \{j\}, |S|=k-1, |F|=l-1\} $ and 
\begin{equation*}
    M_{u}^{i,j}(S,F) = u(S\cup \{i\}, F \cup \{j\})+u(S,F)-u(S\cup \{i\},F)-u(S,F\cup \{j\}).
\end{equation*}
The function $M_{u}^{i,j}$ allows us to quantify how much removing a specific cell at $(i,j)$ from a given set $(S\cup\{i\}, F\cup\{j\})$ affects the overall utility, and \texttt{2D-Shapley} $\psi_{ij}^{\mathrm{2D-Shap}}$ evaluates the average $M_{u}^{i,j}$ across all possible data fragments $(S,F)\in\mathcal{D}_{k,l}^{(i,j)}$.

Similar to \texttt{DataShapley}, the permutation of all rows and columns required for exact \texttt{2D-Shapley} calculations presents significant computational challenges. To address this, \citet{liu20232d} develop \texttt{2D-KNN}, which utilizes $k$-nearest-neighbors models as surrogates to approximate \texttt{2D-Shapley} values. However, the approximation methods can compromise the accuracy of valuations \citep{kwon2023data,jiang2023opendataval}. Additionally, \texttt{2D-KNN} still faces challenges scaling to large-scale datasets and high-dimensional settings.

To address these limitations, we propose \texttt{2D-OOB}, an \textit{efficient} and \textit{model-agnostic} joint valuation framework that leverages out-of-bag estimation to attribute data contribution. We also illustrate how \texttt{2D-OOB} is connected to \texttt{Data-OOB}, thereby facilitating sample-wise feature-level interpretation for data valuation, as discussed in Section \ref{sec:theory}. 

\subsection{2D-OOB: an efficient joint valuation framework}
\label{sec:formalization}

Our idea builds upon the subset bagging model \citep{ho1995random}, which is well recognized as an earlier version of Breiman's random forest model \citep{breiman2001random}. A key distinction from a standard bagging model is that a weak learner in a subset bagging model is trained on a randomly selected subset of features. For $b\in[B]$, we denote the $b$-th random feature subset by $S_b \subseteq [d]$. 
Then, the $b$-th weak learner of a subset bagging model is given as follows.
\begin{equation*}
    \hat{f}_{b} := \mathrm{argmin}_{f} \sum_{i=1} ^n w_{bi} \ell(y_{i} , f(x_{i, S_b})),
\end{equation*}
where $x_{i, S_b}$ is a subvector of $x_i$ that only takes elements in a subset $S_b$. This difference enables us to assess the impact of which features are more influential: if $S_b$ includes a helpful (or detrimental) feature, we can expect the out-of-bag prediction $\hat{f}(x_{i, S_b})$ to be good (or poor). 
We formalize this intuition and propose \texttt{2D-OOB}. For $i \in [n]$, $j \in [d]$ and $ \{(\mathbf{w}_b, S_b, \hat{f}_b)\}_{b=1} ^B$, the \texttt{2D-OOB} for the $j$-th cell of the $i$-th data point is defined as follows,
\begin{align}
    \psi_{ij} ^{\mathrm{2D-OOB}} 
    := \frac{\sum_{b=1} ^B \mathds{1} (w_{bi} =0, j \in S_b ) T(y_i, \hat{f}_b (x_{i, S_b})) }{\sum_{b=1} ^B \mathds{1} (w_{bi} =0, j \in S_b)},
    \label{eqn:df_oob}
\end{align}
where $T: \mathcal{Y} \times  \mathcal{Y} \rightarrow \mathbb{R}$ is a utility function that scores the performance of the weak learner $\hat{f}_b (x_{i, S_b})$ on the $i$-th datum $(x_i,y_i)$. Specifically, for binary or multi-class classification problems, we can adopt $T(y_i, \hat{f}_b (x_{i, S_b})) = \mathds{1}(y_i = \hat{f}_b (x_{i, S_b}))$. In this case, \texttt{2D-OOB} measures the average accuracy score of out-of-bag predictions (specifically, when the $i$-th data point is out-of-bag) if the $j$-th feature is used in training $\hat{f}_b$. For regression problems, we can use the negative squared error loss function, defined as $T(y_i, \hat{f}_b (x _{i, S_b} )) = -(y_i  - \hat{f}_b (x _{i, S_b} ))^2$. In practice, $\mathcal{X}$ could also be incorporated into $T$ to suit the specific use case.

While \texttt{Data-OOB} in~\eqref{eqn:original_OOB} aims to assess the impact of the $i$-th datum, \texttt{2D-OOB} in~\eqref{eqn:df_oob} further provides interpretable insights by evaluating the data point with various combinations of features, revealing which cells are influential to model performance. By leveraging the subset bagging scheme, \texttt{2D-OOB} only requires a single training of the bagging model, making it computationally efficient.

\subsection{Connection to Data-OOB}
\label{sec:theory}

We now present interpretable expressions of how \texttt{2D-OOB} connects to \texttt{Data-OOB} in the following proposition. To begin with, we denote a set of subsets of $[d]$ by $\mathcal{S} := \{S \subseteq [d] \}$. With $\{(\mathbf{w}_b, \hat{f}_b)\}_{b=1} ^B$,
we define the $i$-th \texttt{Data-OOB} when a particular subset $S$ is used as follows and denote it by $\phi_i ^{\mathrm{OOB}} (S)$.
\begin{equation*}
    \phi_i ^{\mathrm{OOB}} (S) := \frac{\sum_{b=1} ^B \mathds{1}(w_{bi}=0 ) T(y_i, \hat{f}_{b} (x_{i, S})) }{ \sum_{b=1} ^B \mathds{1}(w_{bi}=0 ) }.
\end{equation*}

\begin{proposition}
For all $i\in[n]$ and $j\in[d]$, $\psi_{ij} ^{\mathrm{2D-OOB}}$ can be expressed as follows.
\begin{align*}
    \psi_{ij} ^{\mathrm{2D-OOB}} &= \mathbb{E}_{\hat{F}_S} [ \phi_i ^{\mathrm{OOB}} (S) \mid j \in S], 
\end{align*}
where $\hat{F_S}$ is an empirical distribution with respect to $S$ induced by the sampling process. 
\label{prop:representation}
\end{proposition}

A proof is given in Appendix \ref{sec:proof}. Proposition~\ref{prop:representation} shows that \texttt{2D-OOB} $\psi_{ij} ^{\mathrm{2D-OOB}}$ can be expressed as a conditional empirical expectation of \texttt{Data-OOB} provided that the $j$-th feature is used in \texttt{Data-OOB} computation. It provides intuitive interpretations: for a fixed $i$ and $j \neq k$, $\psi_{ij} ^{\mathrm{2D-OOB}} > \psi_{ik} ^{\mathrm{2D-OOB}}$ implies that the cell $x_{ij}$ is more helpful to achieve a higher OOB score than the cell $x_{ik}$, where the OOB score serves as an indicator of model performance. By distinguishing the contributions of individual cells, \texttt{2D-OOB} effectively realizes joint valuation, providing a finer granularity of analysis that links feature-level importance to individual data quality.

\section{Experiments}
\label{sec: experiment}
In this section, we empirically show the effectiveness of \texttt{2D-OOB} across multiple use cases of the joint valuation: \textit{cell-level outlier detection}, \textit{cell fixation}, and \textit{backdoor trigger detection}. To the best of our knowledge, the latter two use cases are introduced here for the first time within the joint valuation framework.
As a summary, \texttt{2D-OOB} can precisely identify anomalous cells that should be prioritized for examination and subsequent fixation to improve model performance. In the context of backdoor trigger detection, \texttt{2D-OOB} demonstrates its efficacy by accurately identifying different types of triggers within poisoned data, showcasing its proficiency in detecting non-random, targeted anomalies. Our method also exhibits high computational efficiency through run-time comparison. 

Throughout all of our experiments, \texttt{2D-OOB} uses a subset bagging model with $B = 1000$ decision trees. We randomly select a fixed ratio of features to build each decision tree. Unless otherwise specified, we utilize half of the features for each weak learner and set $T(y_i, \hat{f}(x_{i, S_b})) = \mathds{1}(y_i = \hat{f}(x_{i, S_b}))$. The run time is measured on a single Intel Xeon Gold 6226 2.9 GHz CPU processor. We provide a detailed ablation study of key hyperparameters in Section \ref{sec:ablation}.

\subsection{Cell-level outlier detection}
\label{sec:cell-level outlier detection}
\paragraph{Experimental setting} 
In practical situations, even when dealing with abnormal data points, it is not always the case that all cells are noisy \citep{rousseeuw2018detecting,liu20232d, kriegel2009outlier}. To simulate more realistic settings, we introduce noise to certain \textit{cells} in the following two-step process: First, we randomly select $20\%$ rows for each dataset. We then select $20\%$ columns uniformly at random, allowing each selected row to have a different set of perturbed cells. We inject noises sampled from the low-probability region into these cells, following \citet{du2022vos} and \citet{liu20232d}. Details on the outlier injection process can be found in Appendix \ref{sec:cell-level outlier generation}. 

We use $12$ publicly accessible binary classification datasets from OpenML, encompassing a range of both low and high-dimensional datasets, which have been widely used in the literature \citep{ghorbani2019data, kwon2021beta, kwon2023data}. Details on these datasets are presented in Appendix \ref{sec:datasets}. For each dataset, $1000$ and $3000$ data points are randomly sampled for training and test datasets, respectively. For the baseline method, we consider \texttt{2D-KNN}, a fast and performant variant of \texttt{2D-Shapley} \citep{liu20232d}. We incorporate a distance regularization term in the utility function $T$ for enhanced performance. 

\begin{figure}
  \centering
  \includegraphics[width = \textwidth]{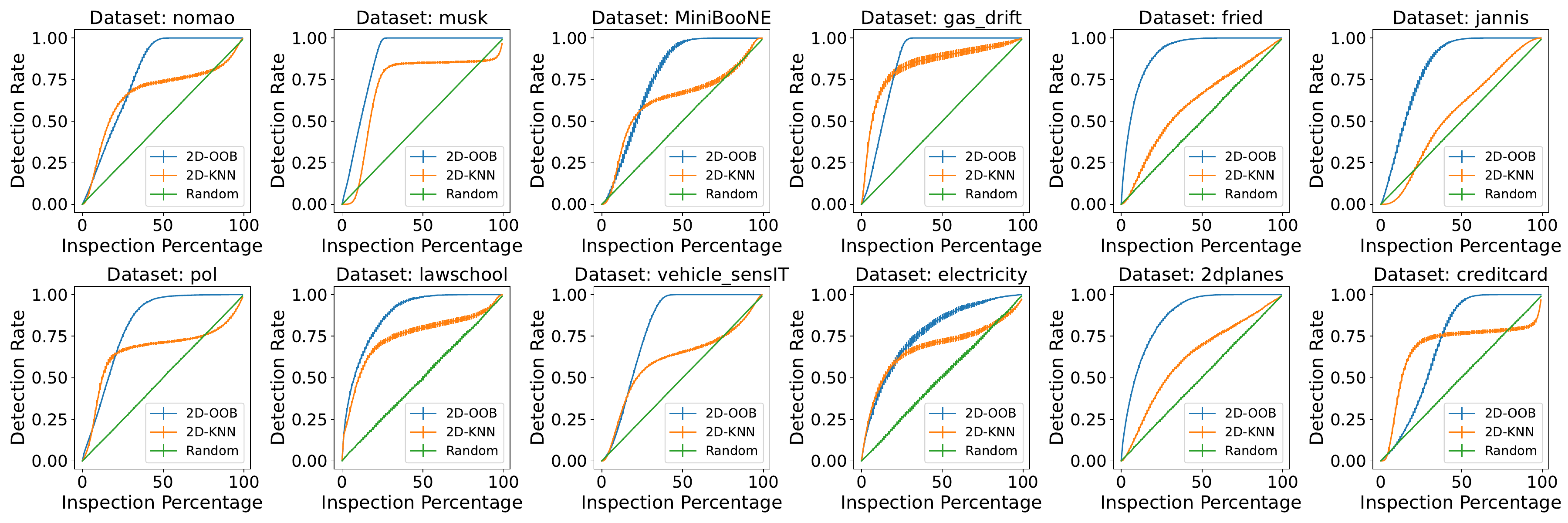}
  \vspace{-6mm}
  \caption{\textbf{Cell-level outlier detection rate curves for \texttt{2D-OOB}, \texttt{2D-KNN}, and \texttt{Random}.} The x-axis represents the percentage of inspected cells. The y-axis represents the detection rate, defined as the ratio of the number of detected outlier cells to the total number of outlier cells present in a dataset. The error bars show a $95\%$ confidence interval based on $30$ independent experiments. We examine the cells in ascending order, starting from those with the lowest values, and thus a curve closer to the left-top corner indicates better performance. \texttt{2D-OOB} efficiently detects the majority of outlier cells by examining only a small fraction of the total cells, while \texttt{2D-KNN} and \texttt{Random} require scanning nearly all the cells.
  }
  \label{Fig:cell_outlier_curve}
\end{figure}

\paragraph{Results} We calculate the valuations for each cell using our joint valuation framework. Ideally, the outlier cells should receive a low valuation. We then arrange the cell valuations in \textit{ascending} order and inspect those cells with the lowest values first. 

The detection rate curve of inserted outlier is shown in Figure \ref{Fig:cell_outlier_curve}. For all datasets, \texttt{2D-OOB} successfully identifies over $90\%$ of the outlier cells by inspecting only $30\%$ of the bottom cells. In comparison, \texttt{2D-KNN} requires examining nearly $90\%$ of the cells to achieve the same detection level.

We also evaluate the area under the curve (AUC) as a quantitative metric and measure the run-time. As Table \ref{tab:outlier} shows, \texttt{2D-OOB} achieves an average AUC of $0.83$ across $12$ datasets, compared to $0.67$ for \texttt{2D-KNN}, while being significantly faster. For high-dimensional datasets such as the musk dataset, which comprises $166$ features, \texttt{2D-KNN} would take more than an hour to process, while \texttt{2D-OOB} can finish in seconds. Furthermore, we present additional results on \textbf{multi-class classification} datasets in Appendix \ref{sec: two-stage}, demonstrating the consistently superior performance and efficiency of \texttt{2D-OOB}.

\begin{table}[t]
\caption{\textbf{Cell-level outlier detection results.} AUC and run-time comparison between \texttt{2D-OOB} and \texttt{2D-KNN} across twelve binary classification datasets. The average and standard error of the AUC and run-time (in seconds) based on $30$ independent experiments are denoted by ``average $\pm$ standard error''. Bold numbers denote the best method. The AUC value for the \texttt{Random} method consistently remains at $0.5$ across all datasets. In every dataset, \texttt{2D-OOB} achieves a significantly higher AUC while being orders of magnitude faster than \texttt{2D-KNN}.
}
\label{tab:outlier}
\begin{center}
\resizebox{0.8\textwidth}{!}{
\begin{tabular}{l|cc|cc}
\toprule
\multirow{2}{*}{Dataset}  & \multicolumn{2}{c}{\textbf{AUC} $\uparrow$} & \multicolumn{2}{|c}{\textbf{Run-time} $\downarrow$} \\
& \texttt{2D-OOB} (ours) & \texttt{2D-KNN} & \texttt{2D-OOB} (ours) & \texttt{2D-KNN} \\
\midrule
lawschool        & \textbf{0.88$\pm$ 0.0027} & 0.75$\pm$ 0.0011 & \textbf{3.33 $\pm$ 0.06}& 177.56 $\pm$ 1.92\\
electricity      & \textbf{0.77$\pm$ 0.0072} & 0.68$\pm$ 0.0014 & \textbf{3.39 $\pm$ 0.07}& 191.38 $\pm$ 2.60\\
fried            & \textbf{0.91$\pm$ 0.0015} & 0.61$\pm$ 0.0005 & \textbf{3.97 $\pm$ 0.10}& 322.79 $\pm$ 2.98\\
2dplanes         & \textbf{0.87$\pm$ 0.0015} & 0.62$\pm$ 0.0005 & \textbf{3.46 $\pm$ 0.05}& 295.25 $\pm$ 2.37\\
creditcard       & \textbf{0.72$\pm$ 0.0028} & 0.69$\pm$ 0.0011 & \textbf{4.56 $\pm$ 0.10}& 662.34 $\pm$ 7.12\\
pol              & \textbf{0.82$\pm$ 0.0014} & 0.67$\pm$ 0.0006 & \textbf{4.34 $\pm$ 0.05}& 759.33 $\pm$ 4.37\\
MiniBooNE        & \textbf{0.77$\pm$ 0.0058} & 0.63$\pm$ 0.0019 & \textbf{7.46 $\pm$ 0.06}& 1507.83 $\pm$ 14.50\\
jannis           & \textbf{0.83$\pm$ 0.0042} & 0.55$\pm$ 0.0004 & \textbf{7.98 $\pm$ 0.07}& 1753.10 $\pm$ 12.35\\
nomao            & \textbf{0.79$\pm$ 0.0021} & 0.67$\pm$ 0.0009 & \textbf{7.69$\pm$ 0.11}& 2564.58 $\pm$ 23.11\\
vehicle\_sensIT  & \textbf{0.81$\pm$ 0.0014} & 0.61$\pm$ 0.0005 & \textbf{9.87 $\pm$ 0.08}& 3113.65 $\pm$ 24.54\\
gas\_drift       & \textbf{0.86$\pm$ 0.0010} & 0.84$\pm$ 0.0017 & \textbf{11.28$\pm$ 0.10}& 3878.31 $\pm$ 40.72\\
musk             & \textbf{0.88$\pm$ 0.0008} & 0.71$\pm$ 0.0006 & \textbf{14.09 $\pm$ 0.11}& 4415.45 $\pm$ 22.96\\
\midrule
Average          & \textbf{0.83} & 0.67 &  \textbf{6.78}& 1636.80\\
\bottomrule
\end{tabular}
}
\end{center}
\vskip -0.1in
\end{table}

\subsection{Cell fixation experiment} 
\label{sec:cell fixation}

\paragraph{Experimental setting} A naive strategy to handle cell-level outliers is to eliminate data points that contain outliers. This method, however, risks substantial data loss, particularly when outliers are scattered and data points are costly to collect. We instead consider a cell fixation experiment, where we assume that the ground-truth annotations of outlier cells can be restored with external expert knowledge.  
At each step, we ``fix" a certain number of cells by substituting them with their ground-truth annotations, prioritizing cells that have the lowest valuations. Then we fit a logistic model\footnote{Logistic regression is chosen because it is a simple yet powerful machine learning model commonly used to test data separability, which is a standard practice in this field \citep{jiang2023opendataval,liu20232d}.} and evaluate the model's performance with a test set of $3000$ samples. It is important to note that correcting normal cells has no effect, whereas fixing outlier cells is expected to enhance the model's performance. We adopt the same datasets and implementations as in Section \ref{sec:cell-level outlier detection}.

\begin{figure}
  \centering
  \includegraphics[width = \textwidth]{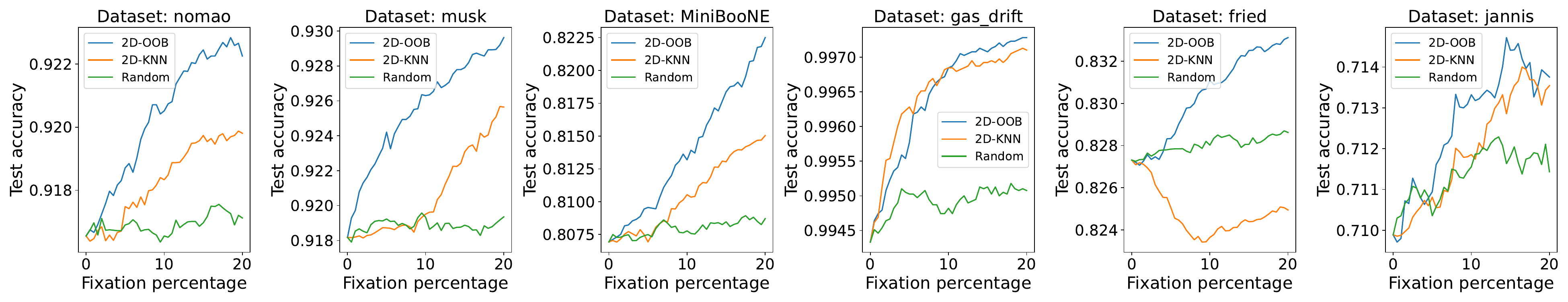}
  \vspace{-6mm}
  \caption{\textbf{Cell fixation experiment results (test accuracy curves) for \texttt{2D-OOB}, \texttt{2D-KNN}, and \texttt{Random}.} We replace cells with their ground-truth annotations, starting with those cells assigned the lowest valuations. The results for $6$ datasets are presented, and additional results for other datasets are available in Appendix \ref{sec:additional fixation}.
  We conduct $30$ independent trials and report the average results. A higher curve indicates better performance. \texttt{2D-OOB} demonstrates a superior capability in accurately identifying and rectifying cell-level outliers.
  }
  \label{Fig:cell_fixation_main}
\end{figure}

\paragraph{Results}  Figure \ref{Fig:cell_fixation_main} illustrates the anticipated trend in the performance of \texttt{2D-OOB}, validating our method's capability to accurately identify and prioritize the most impactful outliers for correction. As cells with the lowest valuations are progressively fixed, \texttt{2D-OOB} demonstrates a consistent improvement in model accuracy. In contrast, when applying the same procedure with \texttt{2D-KNN}, such notable performance enhancements are not observed. 

Additionally, we investigate a scenario where ground-truth annotations remain unavailable. We adopt the setup from \citet{liu20232d}, where we replace the outlier cells with the average of other cells in the same feature column. \texttt{2D-OOB} uniformly demonstrates significant superiority over its counterparts. Results are provided in Appendix \ref{sec:additional fixation}.

\begin{figure}
  \centering
  \includegraphics[width =\textwidth]{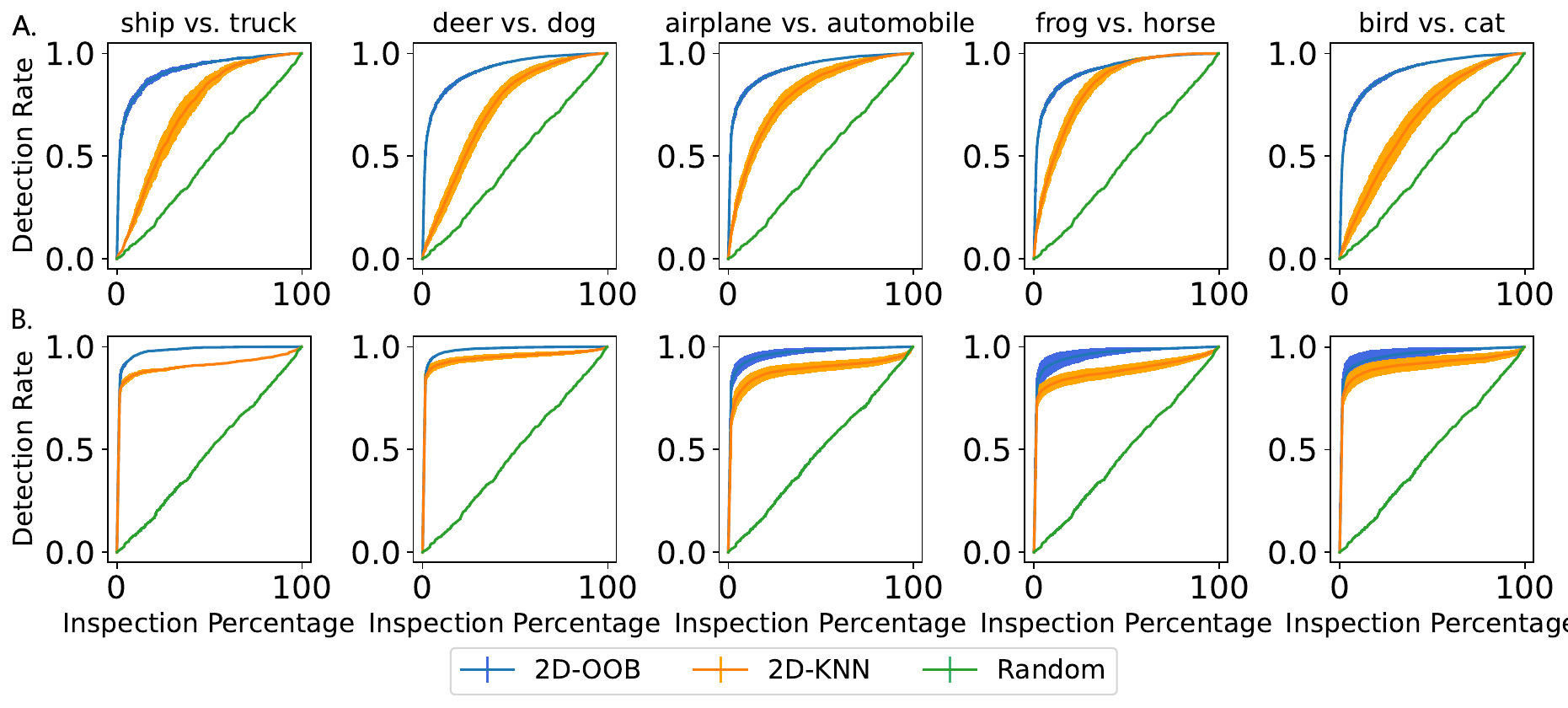}
  \vspace{-3mm}
  \caption{\textbf{Backdoor trigger detection rate curves for \texttt{2D-OOB}, \texttt{2D-KNN}, and \texttt{Random}.} Panels A (top) and B (bottom) correspond to the Trojan square and BadNets square, respectively. We inspect the cells within each poisoned sample in descending order of their valuation scores. The detection rate curve shows the average detection rate across all poisoned samples, with error bars representing a $95\%$ confidence interval based on $15$ independent runs. \texttt{2D-OOB} demonstrates superior performance in detecting the cells implanted with triggers.}
  \label{fig:backdoor_detection_curve}
\end{figure}

\subsection{Backdoor trigger detection} 
\label{sec:backdoor}

A common strategy of data poisoning attacks involves inserting a predefined trigger (e.g., a specific pixel pattern in an image) into a subset of the training data \citep{gu2017badnets, chen2017targeted, Liu2018TrojaningAO}. These malicious manipulations can be challenging to detect as they only infect specific, targeted samples. Even when poisoned data are present, it could be difficult to discern the root cause of the attacks, since manually reviewing the images is expensive and time-consuming. In this experiment, we introduce a novel task enabled by the joint valuation framework: localizing backdoor triggers in data poisoning attacks. Distinct from the random outliers investigated previously, this type of cell contamination is targeted and deliberate. 

\begin{figure}
  \centering
  \includegraphics[width = \textwidth]{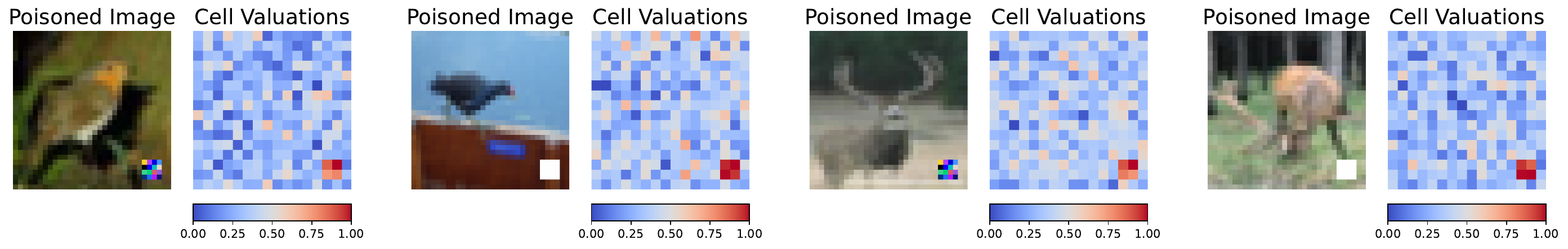}
  \vspace{-6mm}
  \caption{\textbf{Qualitative examples for \texttt{2D-OOB} in the backdoor trigger detection task.} Each pair of images consists of a poisoned image and its corresponding cell valuation heatmap. The color of the heatmap indicates importance, with red cells representing higher importance and blue cells representing lower importance. In the first two pairs, the class `bird'' is relabeled as ``cat'', while in the latter two pairs, the class ``deer'' is relabeled as ``cat''. The heatmaps clearly show that higher cell valuations are predominantly concentrated in the regions containing triggers, while areas featuring actual objects receive lower valuations. This pattern suggests that \texttt{2D-OOB} effectively captures the triggers as the impactful features responsible for the misclassification of the poisoned samples. 
  }
  \label{fig:backdoor_heatmap}
\end{figure}

We consider two popular backdoor attack algorithms: BadNets \citep{gu2017badnets} and Trojan Attack \citep{Liu2018TrojaningAO}. During training, the poisoned samples, relabeled as the adversarial target class, are mixed up with clean data. As a result, the model learns to incorrectly associate the trigger with the target class.
At test time, inputs containing the trigger are misclassified to the target class. In this context, our goal is to effectively pinpoint the triggers by recognizing them as influential features through our joint valuation framework.

\textbf{Experimental setting} We select 5 pairs of classes from CIFAR-10 \citep{krizhevsky2009learning}. For each pair, one class is designated as the target attack class, while the other serves as the source class. The training dataset comprises $1000$ images. For each attack, we contaminate $15\%$ of the training samples from the source class and relabel them to the target class. Two types of attack triggers are implemented: the Trojan square and the BadNets square \citep{gu2017badnets,pang:2022:eurosp, Liu2018TrojaningAO}. These triggers are placed in the lower right corner of the original images to minimize occlusion. Further details about these attacks are available in Appendix \ref{sec:backdoor trigger generation}. In our experiment, the ratio of poisoned cells is approximately $1\%$, and each weak learner in the subset bagging model is built by sampling $25$\% of the features.



\paragraph{Results} We adopt the same detection scheme and baseline methods as in Section~\ref{sec:cell-level outlier detection}. Ideally, the poisoned cells should receive high valuation scores given that such data points have been relabeled. We present the detection rate curves for five datasets in Figure \ref{fig:backdoor_detection_curve}. \texttt{2D-OOB} significantly outperforms \texttt{2D-KNN} in detecting both types of triggers. Overall, \texttt{2D-OOB} achieves an average detection AUC of $0.95$ across all datasets and attack types, compared to $0.83$ for \texttt{2D-KNN}. It is worth noting that conventional data valuation methods are fundamentally unable to localize backdoor triggers; at most, they can only identify poisoned data points.

\paragraph{Qualitative examples} Figure \ref{fig:backdoor_heatmap} displays the heatmaps for poisoned samples based on cell valuations of \texttt{2D-OOB}. Areas with higher cell valuations (marked in dark red) precisely indicate the trigger locations within these samples, demonstrating the effectiveness of our joint valuation framework. Additional examples can be found in Appendix \ref{sec:additional backdoor}.



\subsection{Ablation study}
\label{sec:ablation}
We conduct ablation studies on the cell-level outlier detection task, as outlined in Section \ref{sec:cell-level outlier detection}, to examine the impact of key hyperparameters on \texttt{2D-OOB} estimations, including the selection and number of weak learners, as well as the feature subset ratio.

\paragraph{Selection of weak learners} 
Although our study primarily employs decision trees as weak learners, it is important to note that \texttt{2D-OOB} is \textbf{model-agnostic}, enabling the use of any class of machine learning models as weak learners. Specifically, we examine decision trees, logistic regression, a single-layer MLP with $64$ units, and a two-layer MLP with $64$ and $32$ units, respectively, as weak learners to compare their performance.

Table \ref{tab:ablation_weak_learner_type} presents a comparison of detection AUC across $12$ datasets with different choices of weak learners, indicating that \texttt{2D-OOB} is not model-free. The selection of weak learners \textit{slightly} affects the valuation results, with more complex models generally yielding better performance. Nonetheless, all variations of \texttt{2D-OOB} outperform \texttt{2D-KNN}, highlighting the clear advantages of the \texttt{2D-OOB} approach.

\begin{table}
\caption{\textbf{Ablation study results of weak learner types.} The average and standard error of the detection AUC based on $30$ independent experiments are denoted by ``average $\pm$ standard error''. Results from \texttt{2D-KNN} are included for comparison. The choice of weak learner leads to variations in cell values, yet the performance of the detection task remains robust.
}
\label{tab:ablation_weak_learner_type}
\begin{center}
\resizebox{\textwidth}{!}{
\begin{tabular}{l|ccccc}
\toprule
Dataset& Decision Tree & Logistic Regression & MLP (single-layer) & MLP (two-layer) & 2D-KNN (Baseline)\\
\midrule
lawschool        & \textbf{0.88 $\pm$ 0.0027} & 0.81 $\pm$ 0.0014 & 0.83 $\pm$ 0.0023 & 0.86 $\pm$ 0.0049 & 0.75 $\pm$ 0.0011 \\
electricity      & \textbf{0.77 $\pm$ 0.0072} & 0.75 $\pm$ 0.0029 & 0.75 $\pm$ 0.0039 & 0.74 $\pm$ 0.0064 & 0.68 $\pm$ 0.0014 \\
fried            & \textbf{0.91 $\pm$ 0.0015} & 0.82 $\pm$ 0.0023 & 0.85 $\pm$ 0.0020 & 0.88 $\pm$ 0.0027 & 0.61 $\pm$ 0.0005 \\
2dplanes         & 0.87 $\pm$ 0.0015 & 0.82 $\pm$ 0.0026 & 0.86 $\pm$ 0.0026 & \textbf{0.88 $\pm$ 0.0037} & 0.62 $\pm$ 0.0005 \\
creditcard       & 0.72 $\pm$ 0.0028 & \textbf{0.74 $\pm$ 0.0023} & \textbf{0.74 $\pm$ 0.0026} & \textbf{0.74 $\pm$ 0.0071} & 0.69 $\pm$ 0.0011 \\
pol              & 0.82 $\pm$ 0.0014 & 0.79 $\pm$ 0.0029 & 0.85 $\pm$ 0.0014 & \textbf{0.86 $\pm$ 0.0019} & 0.67 $\pm$ 0.0006 \\
MiniBooNE        & 0.77 $\pm$ 0.0058 & 0.77 $\pm$ 0.0059 & 0.80 $\pm$ 0.0057 & \textbf{0.81 $\pm$ 0.0119} & 0.63 $\pm$ 0.0019 \\
jannis           & \textbf{0.83 $\pm$ 0.0042} & 0.76 $\pm$ 0.0040 & 0.79 $\pm$ 0.0048 & 0.80 $\pm$ 0.0108 & 0.55 $\pm$ 0.0004 \\
nomao            & 0.79 $\pm$ 0.0021 & 0.82 $\pm$ 0.0012 & \textbf{0.83 $\pm$ 0.0010} & \textbf{0.83 $\pm$ 0.0017} & 0.67 $\pm$ 0.0009 \\
vehicle-sensIT   & 0.81 $\pm$ 0.0014 & 0.81 $\pm$ 0.0026 & 0.80 $\pm$ 0.0025 & \textbf{0.82 $\pm$ 0.0037} & 0.61 $\pm$ 0.0005 \\
gas-drift        & 0.86 $\pm$ 0.0010 & \textbf{0.89 $\pm$ 0.0005} & 0.88 $\pm$ 0.0005 & 0.88 $\pm$ 0.0006 & 0.84 $\pm$ 0.0017 \\
musk             & 0.88 $\pm$ 0.0008 & 0.87 $\pm$ 0.0005 & \textbf{0.88 $\pm$ 0.0005} & \textbf{0.88 $\pm$ 0.0008} & 0.71 $\pm$ 0.0006\\
\midrule
Average & \textbf{0.83} & 0.80 & 0.82 & \textbf{0.83} & 0.67 \\
\bottomrule
\end{tabular}
}
\end{center}
\vskip -0.1in
\end{table}

\paragraph{Number of weak learners}
Increasing the number of weak learners allows for a greater number of data-feature subset pairs to be explored, potentially leading to more accurate estimates. However, as shown in Table \ref{tab:ablation_B},
when we increase the number of base learners from $500$ to $3000$, the detection AUC for each dataset remains relatively unchanged, suggesting convergence of the estimation beyond a certain threshold. Typically, $1000$ base learners are sufficient to achieve an equitable joint valuation.

\paragraph{Feature subset ratio $K/d$} The feature subset ratio $K/d$ refers to the fraction of the total number of features $d$ that are randomly selected to build each weak learner, where $K$ is the number of selected features. In previous experiments, we used a fixed ratio of $0.50$ (unless otherwise specified). To further investigate the impact of this ratio, we now test two additional values: $0.25$ and $0.75$. The results in Table \ref{tab:ablation_K} suggest that in general, the joint valuation capacity of our method is robust to the choice of feature subset ratio. 

\begin{table*}[t]
\caption{\textbf{Ablation study results of (a) the number of base learners \( B \) and (b) the feature subset ratio \( K/d \).} The average and standard error of the detection AUC based on $30$ independent runs are denoted by ``average $\pm$ standard error.'' (a) Increasing the number of base learners from $1000$ to $3000$ does not yield a notable performance improvement. (b) Our method's joint valuation capacity remains relatively stable regardless of the selected feature subset ratio.}
\label{tab:combined_ablation}
\centering
\resizebox{\textwidth}{!}{
\begin{tabular}{cc}
    \begin{subtable}[t]{0.5\textwidth}
    \caption{\textbf{Ablation on the number of base learners $B$.} }
\label{tab:ablation_B}
\resizebox{\linewidth}{!}{
\begin{tabular}{l|c|c|c}
\toprule
\multirow{2}{*}{Dataset}  & \multicolumn{2}{c}{\textbf{AUC} $\uparrow$} \\
& $B = 500$ &  $B = 1000$ &  $B= 3000$ \\
\midrule
lawschool        & 0.86 $\pm$ 0.0035    & 0.88 $\pm$ 0.0027     & 0.88 $\pm$ 0.0026     \\ 
electricity      & 0.77 $\pm$ 0.0062    & 0.77 $\pm$ 0.0072     & 0.77 $\pm$ 0.0070     \\ 
fried            & 0.87 $\pm$ 0.0022    & 0.91 $\pm$ 0.0015     & 0.91 $\pm$ 0.0014     \\ 
2dplanes         & 0.87 $\pm$ 0.0016    & 0.87 $\pm$ 0.0015     & 0.87 $\pm$ 0.0015     \\ 
creditcard       & 0.72 $\pm$ 0.0025    & 0.72 $\pm$ 0.0028     & 0.72 $\pm$ 0.0028     \\ 
pol              & 0.78 $\pm$ 0.0022    & 0.82 $\pm$ 0.0014     & 0.82 $\pm$ 0.0014     \\ 
MiniBooNE        & 0.77 $\pm$ 0.0042    & 0.77 $\pm$ 0.0058     & 0.77 $\pm$ 0.0058     \\ 
jannis           & 0.78 $\pm$ 0.0045    & 0.83 $\pm$ 0.0042     & 0.83 $\pm$ 0.0039     \\ 
nomao            & 0.79 $\pm$ 0.0018    & 0.79 $\pm$ 0.0021     & 0.79 $\pm$ 0.0020     \\ 
vehicle\_sensIT   & 0.80 $\pm$ 0.0021    & 0.81 $\pm$ 0.0014     & 0.81 $\pm$ 0.0014     \\ 
gas\_drift        & 0.86 $\pm$ 0.0007    & 0.86 $\pm$ 0.0010     & 0.86 $\pm$ 0.0010     \\ 
musk             & 0.88 $\pm$ 0.0008    & 0.88 $\pm$ 0.0008     & 0.88 $\pm$ 0.0008\\
\midrule
Average       &0.81   & \textbf{0.83}  & \textbf{0.83} \\
\bottomrule
\end{tabular}
}
    \end{subtable}
    &
    \begin{subtable}[t]{0.5\textwidth}
        \centering
\caption{\textbf{Ablation on feature subset ratio $K/d$.} }
\label{tab:ablation_K}
\resizebox{\linewidth}{!}{
\begin{tabular}{l|c|c|c}
\toprule
\multirow{2}{*}{Dataset}  & \multicolumn{3}{c}{\textbf{AUC} $\uparrow$} \\
& $K/d = 0.25$ & $K/d = 0.50$ & $K/d = 0.75$ \\
\midrule
lawschool        & 0.86$\pm$ 0.0026 & 0.88$\pm$ 0.0027 & 0.88$\pm$ 0.0024 \\
electricity      & 0.79$\pm$ 0.0070 & 0.77$\pm$ 0.0072 & 0.73$\pm$ 0.0070 \\
fried            & 0.86$\pm$ 0.0024 & 0.91$\pm$ 0.0015 & 0.89$\pm$ 0.0007 \\
2dplanes         & 0.82$\pm$ 0.0015 & 0.87$\pm$ 0.0015 & 0.88$\pm$ 0.0014 \\
creditcard       & 0.73$\pm$ 0.0029 & 0.72$\pm$ 0.0028 & 0.71$\pm$ 0.0028 \\
pol              & 0.66$\pm$ 0.0031 & 0.82$\pm$ 0.0014 & 0.82$\pm$ 0.0014 \\
MiniBooNE        & 0.78$\pm$ 0.0076 & 0.77$\pm$ 0.0058 & 0.77$\pm$ 0.0049 \\
jannis           & 0.84$\pm$ 0.0035 & 0.83$\pm$ 0.0042 & 0.82$\pm$ 0.0043 \\
nomao            & 0.79$\pm$ 0.0019 & 0.79$\pm$ 0.0021 & 0.78$\pm$ 0.0021 \\
vehicle\_sensIT  & 0.81$\pm$ 0.0014 & 0.81$\pm$ 0.0014 & 0.80$\pm$ 0.0015 \\
gas\_drift       & 0.88$\pm$ 0.0009 & 0.86$\pm$ 0.0010 & 0.86$\pm$ 0.0009 \\
musk             & 0.89$\pm$ 0.0008 & 0.88$\pm$ 0.0008 & 0.88$\pm$ 0.0008 \\
\midrule
Average          & 0.81 & \textbf{0.83} & 0.82 \\
\bottomrule
\end{tabular}
}
    \end{subtable}
\end{tabular}
}
\end{table*}

Apart from the experiments discussed above, we showcase that marginalization of \texttt{2D-OOB} can either match or surpass state-of-the-art data valuation methods on standard benchmarks in Appendix \ref{sec:data val exp}. 

\section{Related work}
\paragraph{Data contribution estimation} 
In addition to the marginal contribution-based methods discussed in Section \ref{sec: marginal}, many other approaches are emerging in the area of data valuation. \citet{just2023lava} develop a non-conventional class-wise Wasserstein distance between the training and validation sets and use the gradient information to evaluate each data point.
\citet{wu2022davinz} extend data valuation to deep neural networks, introducing a training-free data valuation framework based on neural tangent kernel theory. 
\citet{yoon2020data} leverage reinforcement learning techniques to automatically learn data valuation scores by training a regression model.
However, all these data valuation methods do not assign importance scores to cells, whereas our method provides additional insights into how individual cells contribute to the data valuations.

\paragraph{Feature attribution} 
Feature attribution is a pivotal research domain in explainable machine learning that primarily aims to provide insights into how individual features influence model predictions. Various effective methods have been proposed, including SHAP-based explanation \citep{lundberg2017unified,lundberg2018consistent,kwon2022weightedshap,covert2020understanding,covert2021improving}, counterfactual explanation \citep{wachter2017counterfactual,joshi2019towards,poyiadzi2020face,mahajan2019preserving,mothilal2020explaining}, and backpropagation-based explanation \citep{ancona2017towards,bach2015pixel,springenberg2014striving,simonyan2013deep,zeiler2014visualizing}. Among these methods, the SHAP-based explanation stands out as the most widely adopted approach, utilizing cooperative game theory principles to compute the Shapley value \citep{shapley1953value}. 
While feature attribution offers a potential method to attribute data valuation scores across individual cells, our empirical experiments in Appendix \ref{sec: two-stage} reveal that this two-stage scheme falls short in efficacy compared to our proposed joint valuation paradigm, which integrates data valuation and feature attribution in a simultaneous process.

\section{Conclusion}
\label{sec:conclusion}
We propose \texttt{2D-OOB}, an efficient joint valuation framework that assigns a score to each cell in a dataset, thereby facilitating finer attribution of data contributions and enabling a deeper understanding of the dataset. Through comprehensive experiments, we show that \texttt{2D-OOB} is computationally efficient and competitive over state-of-the-art methods in multiple joint valuation use cases. 

\paragraph{Discussion} We emphasize that the primary objective of the joint valuation framework is to evaluate the quality of cells within the dataset, rather than to optimize model performance. The model used serves as a proxy for this evaluation, and it is important to note that a high-performing machine learning method does not necessarily ensure a justified valuation framework.

\paragraph{Limitation and future work}While our study primarily explores random forest models applied to tabular datasets and simple image datasets, the potential application of neural network models within the \texttt{2D-OOB} framework for more complex vision and language tasks presents a promising avenue for future investigation. For instance, in text datasets, tokens or words can be treated as cells. \texttt{2D-OOB} can be easily integrated into any bagging training scheme that uses language models.

Overall, we believe that our work will inspire further exploration in the field of joint valuation, with the broader goal of improving the transparency and interpretability of machine learning, as well as developing an equitable incentive mechanism for data sharing.

\section*{Acknowledgement}
We acknowledge computing resources from Columbia University's Shared Research Computing Facility project, which is supported by NIH Research Facility Improvement Grant 1G20RR030893-01, and associated funds from the New York State Empire State Development, Division of Science Technology and Innovation (NYSTAR) Contract C090171, both awarded April 15, 2010.


\newpage

\bibliographystyle{plainnat}
\small\bibliography{ref}

\begin{thebibliography}{53}
\providecommand{\natexlab}[1]{#1}
\providecommand{\url}[1]{\texttt{#1}}
\expandafter\ifx\csname urlstyle\endcsname\relax
  \providecommand{\doi}[1]{doi: #1}\else
  \providecommand{\doi}{doi: \begingroup \urlstyle{rm}\Url}\fi

\bibitem[Ancona et~al.(2017)Ancona, Ceolini, {\"O}ztireli, and Gross]{ancona2017towards}
Marco Ancona, Enea Ceolini, Cengiz {\"O}ztireli, and Markus Gross.
\newblock Towards better understanding of gradient-based attribution methods for deep neural networks.
\newblock \emph{arXiv preprint arXiv:1711.06104}, 2017.

\bibitem[Bach et~al.(2015)Bach, Binder, Montavon, Klauschen, M{\"u}ller, and Samek]{bach2015pixel}
Sebastian Bach, Alexander Binder, Gr{\'e}goire Montavon, Frederick Klauschen, Klaus-Robert M{\"u}ller, and Wojciech Samek.
\newblock On pixel-wise explanations for non-linear classifier decisions by layer-wise relevance propagation.
\newblock \emph{PloS one}, 10\penalty0 (7):\penalty0 e0130140, 2015.

\bibitem[Bleiholder and Naumann(2009)]{bleiholder2009data}
Jens Bleiholder and Felix Naumann.
\newblock Data fusion.
\newblock \emph{ACM computing surveys (CSUR)}, 41\penalty0 (1):\penalty0 1--41, 2009.

\bibitem[Breiman(2001)]{breiman2001random}
Leo Breiman.
\newblock Random forests.
\newblock \emph{Machine learning}, 45:\penalty0 5--32, 2001.

\bibitem[Chen et~al.(2017)Chen, Liu, Li, Lu, and Song]{chen2017targeted}
Xinyun Chen, Chang Liu, Bo~Li, Kimberly Lu, and Dawn Song.
\newblock Targeted backdoor attacks on deep learning systems using data poisoning.
\newblock \emph{arXiv preprint arXiv:1712.05526}, 2017.

\bibitem[Covert and Lee(2021)]{covert2021improving}
Ian Covert and Su-In Lee.
\newblock Improving kernelshap: Practical shapley value estimation using linear regression.
\newblock In \emph{International Conference on Artificial Intelligence and Statistics}, pages 3457--3465. PMLR, 2021.

\bibitem[Covert et~al.(2020)Covert, Lundberg, and Lee]{covert2020understanding}
Ian Covert, Scott~M Lundberg, and Su-In Lee.
\newblock Understanding global feature contributions with additive importance measures.
\newblock \emph{Advances in Neural Information Processing Systems}, 33:\penalty0 17212--17223, 2020.

\bibitem[Du et~al.(2022)Du, Wang, Cai, and Li]{du2022vos}
Xuefeng Du, Zhaoning Wang, Mu~Cai, and Yixuan Li.
\newblock Vos: Learning what you don't know by virtual outlier synthesis.
\newblock \emph{arXiv preprint arXiv:2202.01197}, 2022.

\bibitem[Feldman and Zhang(2020)]{feldman2020neural}
Vitaly Feldman and Chiyuan Zhang.
\newblock What neural networks memorize and why: Discovering the long tail via influence estimation.
\newblock \emph{Advances in Neural Information Processing Systems}, 33:\penalty0 2881--2891, 2020.

\bibitem[Fernandez et~al.(2020)Fernandez, Subramaniam, and Franklin]{fernandez2020data}
Raul~Castro Fernandez, Pranav Subramaniam, and Michael~J Franklin.
\newblock Data market platforms: Trading data assets to solve data problems.
\newblock \emph{arXiv preprint arXiv:2002.01047}, 2020.

\bibitem[Feurer et~al.(2021)Feurer, Van~Rijn, Kadra, Gijsbers, Mallik, Ravi, M{\"u}ller, Vanschoren, and Hutter]{feurer2021openml}
Matthias Feurer, Jan~N Van~Rijn, Arlind Kadra, Pieter Gijsbers, Neeratyoy Mallik, Sahithya Ravi, Andreas M{\"u}ller, Joaquin Vanschoren, and Frank Hutter.
\newblock Openml-python: an extensible python api for openml.
\newblock \emph{The Journal of Machine Learning Research}, 22\penalty0 (1):\penalty0 4573--4577, 2021.

\bibitem[Ghorbani and Zou(2019)]{ghorbani2019data}
Amirata Ghorbani and James Zou.
\newblock Data shapley: Equitable valuation of data for machine learning.
\newblock In \emph{International Conference on Machine Learning}, pages 2242--2251. PMLR, 2019.

\bibitem[Gu et~al.(2017)Gu, Dolan-Gavitt, and Garg]{gu2017badnets}
Tianyu Gu, Brendan Dolan-Gavitt, and Siddharth Garg.
\newblock Badnets: Identifying vulnerabilities in the machine learning model supply chain.
\newblock \emph{arXiv preprint arXiv:1708.06733}, 2017.

\bibitem[Ho(1995)]{ho1995random}
Tin~Kam Ho.
\newblock Random decision forests.
\newblock In \emph{Proceedings of 3rd international conference on document analysis and recognition}, volume~1, pages 278--282. IEEE, 1995.

\bibitem[Jia et~al.(2019)Jia, Dao, Wang, Hubis, Gurel, Li, Zhang, Spanos, and Song]{jia2019efficient}
Ruoxi Jia, David Dao, Boxin Wang, Frances~Ann Hubis, Nezihe~Merve Gurel, Bo~Li, Ce~Zhang, Costas~J Spanos, and Dawn Song.
\newblock Efficient task-specific data valuation for nearest neighbor algorithms.
\newblock \emph{arXiv preprint arXiv:1908.08619}, 2019.

\bibitem[Jiang et~al.(2023)Jiang, Liang, Zou, and Kwon]{jiang2023opendataval}
Kevin~Fu Jiang, Weixin Liang, James Zou, and Yongchan Kwon.
\newblock Opendataval: a unified benchmark for data valuation.
\newblock \emph{arXiv preprint arXiv:2306.10577}, 2023.

\bibitem[Joshi et~al.(2019)Joshi, Koyejo, Vijitbenjaronk, Kim, and Ghosh]{joshi2019towards}
Shalmali Joshi, Oluwasanmi Koyejo, Warut Vijitbenjaronk, Been Kim, and Joydeep Ghosh.
\newblock Towards realistic individual recourse and actionable explanations in black-box decision making systems.
\newblock \emph{arXiv preprint arXiv:1907.09615}, 2019.

\bibitem[Just et~al.(2023)Just, Kang, Wang, Zeng, Ko, Jin, and Jia]{just2023lava}
Hoang~Anh Just, Feiyang Kang, Jiachen~T Wang, Yi~Zeng, Myeongseob Ko, Ming Jin, and Ruoxi Jia.
\newblock Lava: Data valuation without pre-specified learning algorithms.
\newblock \emph{arXiv preprint arXiv:2305.00054}, 2023.

\bibitem[Kelly et~al.(2017)Kelly, Longjohn, and Nottingham]{kelly_longjohn_nottingham}
Markelle Kelly, Rachel Longjohn, and Kolby Nottingham.
\newblock \url{https://archive.ics.uci.edu}, 2017.
\newblock The UCI Machine Learning Repository.

\bibitem[Koh and Liang(2017)]{koh2017understanding}
Pang~Wei Koh and Percy Liang.
\newblock Understanding black-box predictions via influence functions.
\newblock In \emph{International conference on machine learning}, pages 1885--1894. PMLR, 2017.

\bibitem[Kriegel et~al.(2009)Kriegel, Kr{\"o}ger, Schubert, and Zimek]{kriegel2009outlier}
Hans-Peter Kriegel, Peer Kr{\"o}ger, Erich Schubert, and Arthur Zimek.
\newblock Outlier detection in axis-parallel subspaces of high dimensional data.
\newblock In \emph{Advances in Knowledge Discovery and Data Mining: 13th Pacific-Asia Conference, PAKDD 2009 Bangkok, Thailand, April 27-30, 2009 Proceedings 13}, pages 831--838. Springer, 2009.

\bibitem[Krizhevsky et~al.(2009)Krizhevsky, Hinton, et~al.]{krizhevsky2009learning}
Alex Krizhevsky, Geoffrey Hinton, et~al.
\newblock Learning multiple layers of features from tiny images.
\newblock 2009.

\bibitem[Kwon and Zou(2021)]{kwon2021beta}
Yongchan Kwon and James Zou.
\newblock Beta shapley: a unified and noise-reduced data valuation framework for machine learning.
\newblock \emph{arXiv preprint arXiv:2110.14049}, 2021.

\bibitem[Kwon and Zou(2023)]{kwon2023data}
Yongchan Kwon and James Zou.
\newblock Data-{OOB}: Out-of-bag estimate as a simple and efficient data value.
\newblock In Andreas Krause, Emma Brunskill, Kyunghyun Cho, Barbara Engelhardt, Sivan Sabato, and Jonathan Scarlett, editors, \emph{Proceedings of the 40th International Conference on Machine Learning}, volume 202 of \emph{Proceedings of Machine Learning Research}, pages 18135--18152. PMLR, 23--29 Jul 2023.
\newblock URL \url{https://proceedings.mlr.press/v202/kwon23e.html}.

\bibitem[Kwon and Zou(2022)]{kwon2022weightedshap}
Yongchan Kwon and James~Y Zou.
\newblock Weightedshap: analyzing and improving shapley based feature attributions.
\newblock \emph{Advances in Neural Information Processing Systems}, 35:\penalty0 34363--34376, 2022.

\bibitem[Kwon et~al.(2021)Kwon, A.~Rivas, and Zou]{pmlr-v130-kwon21a}
Yongchan Kwon, Manuel A.~Rivas, and James Zou.
\newblock Efficient computation and analysis of distributional shapley values.
\newblock In Arindam Banerjee and Kenji Fukumizu, editors, \emph{Proceedings of The 24th International Conference on Artificial Intelligence and Statistics}, volume 130 of \emph{Proceedings of Machine Learning Research}, pages 793--801. PMLR, 13--15 Apr 2021.
\newblock URL \url{https://proceedings.mlr.press/v130/kwon21a.html}.

\bibitem[Leung et~al.(2016)Leung, Zhang, and Zamar]{leung2016robust}
Andy Leung, Hongyang Zhang, and Ruben Zamar.
\newblock Robust regression estimation and inference in the presence of cellwise and casewise contamination.
\newblock \emph{Computational Statistics \& Data Analysis}, 99:\penalty0 1--11, 2016.

\bibitem[Liang et~al.(2022)Liang, Tadesse, Ho, Fei-Fei, Zaharia, Zhang, and Zou]{liang2022advances}
Weixin Liang, Girmaw~Abebe Tadesse, Daniel Ho, L~Fei-Fei, Matei Zaharia, Ce~Zhang, and James Zou.
\newblock Advances, challenges and opportunities in creating data for trustworthy ai.
\newblock \emph{Nature Machine Intelligence}, 4\penalty0 (8):\penalty0 669--677, 2022.

\bibitem[Liu et~al.(2018)Liu, Ma, Aafer, Lee, Zhai, Wang, and Zhang]{Liu2018TrojaningAO}
Yingqi Liu, Shiqing Ma, Yousra Aafer, Wen-Chuan Lee, Juan Zhai, Weihang Wang, and X.~Zhang.
\newblock Trojaning attack on neural networks.
\newblock In \emph{Network and Distributed System Security Symposium}, 2018.
\newblock URL \url{https://api.semanticscholar.org/CorpusID:31806516}.

\bibitem[Liu et~al.(2023)Liu, Just, Chang, Chen, and Jia]{liu20232d}
Zhihong Liu, Hoang~Anh Just, Xiangyu Chang, Xi~Chen, and Ruoxi Jia.
\newblock 2d-shapley: A framework for fragmented data valuation.
\newblock \emph{arXiv preprint arXiv:2306.10473}, 2023.

\bibitem[Lundberg and Lee(2017)]{lundberg2017unified}
Scott~M Lundberg and Su-In Lee.
\newblock A unified approach to interpreting model predictions.
\newblock \emph{Advances in neural information processing systems}, 30, 2017.

\bibitem[Lundberg et~al.(2018)Lundberg, Erion, and Lee]{lundberg2018consistent}
Scott~M Lundberg, Gabriel~G Erion, and Su-In Lee.
\newblock Consistent individualized feature attribution for tree ensembles.
\newblock \emph{arXiv preprint arXiv:1802.03888}, 2018.

\bibitem[Mahajan et~al.(2019)Mahajan, Tan, and Sharma]{mahajan2019preserving}
Divyat Mahajan, Chenhao Tan, and Amit Sharma.
\newblock Preserving causal constraints in counterfactual explanations for machine learning classifiers.
\newblock \emph{arXiv preprint arXiv:1912.03277}, 2019.

\bibitem[Mothilal et~al.(2020)Mothilal, Sharma, and Tan]{mothilal2020explaining}
Ramaravind~K Mothilal, Amit Sharma, and Chenhao Tan.
\newblock Explaining machine learning classifiers through diverse counterfactual explanations.
\newblock In \emph{Proceedings of the 2020 conference on fairness, accountability, and transparency}, pages 607--617, 2020.

\bibitem[Nicolae et~al.(2018)Nicolae, Sinn, Tran, Buesser, Rawat, Wistuba, Zantedeschi, Baracaldo, Chen, Ludwig, Molloy, and Edwards]{art2018}
Maria-Irina Nicolae, Mathieu Sinn, Minh~Ngoc Tran, Beat Buesser, Ambrish Rawat, Martin Wistuba, Valentina Zantedeschi, Nathalie Baracaldo, Bryant Chen, Heiko Ludwig, Ian Molloy, and Ben Edwards.
\newblock Adversarial robustness toolbox v1.2.0.
\newblock \emph{CoRR}, 1807.01069, 2018.
\newblock URL \url{https://arxiv.org/pdf/1807.01069}.

\bibitem[Pang et~al.(2022)Pang, Zhang, Gao, Xi, Ji, Cheng, and Wang]{pang:2022:eurosp}
Ren Pang, Zheng Zhang, Xiangshan Gao, Zhaohan Xi, Shouling Ji, Peng Cheng, and Ting Wang.
\newblock Trojanzoo: Towards unified, holistic, and practical evaluation of neural backdoors.
\newblock In \emph{Proceedings of IEEE European Symposium on Security and Privacy (Euro S\&P)}, 2022.

\bibitem[Poyiadzi et~al.(2020)Poyiadzi, Sokol, Santos-Rodriguez, De~Bie, and Flach]{poyiadzi2020face}
Rafael Poyiadzi, Kacper Sokol, Raul Santos-Rodriguez, Tijl De~Bie, and Peter Flach.
\newblock Face: feasible and actionable counterfactual explanations.
\newblock In \emph{Proceedings of the AAAI/ACM Conference on AI, Ethics, and Society}, pages 344--350, 2020.

\bibitem[Rousseeuw and Bossche(2018)]{rousseeuw2018detecting}
Peter~J Rousseeuw and Wannes Van~Den Bossche.
\newblock Detecting deviating data cells.
\newblock \emph{Technometrics}, 60\penalty0 (2):\penalty0 135--145, 2018.

\bibitem[Shapley et~al.(1953)]{shapley1953value}
Lloyd~S Shapley et~al.
\newblock A value for n-person games.
\newblock 1953.

\bibitem[Sim et~al.(2022)Sim, Xu, and Low]{sim2022data}
Rachael Hwee~Ling Sim, Xinyi Xu, and Bryan Kian~Hsiang Low.
\newblock Data valuation in machine learning:“ingredients”, strategies, and open challenges.
\newblock In \emph{Proc. IJCAI}, 2022.

\bibitem[Sim et~al.(2023)Sim, Zhang, Hoang, Xu, Low, and Jaillet]{sim2023incentives}
Rachael Hwee~Ling Sim, Yehong Zhang, Trong~Nghia Hoang, Xinyi Xu, Bryan Kian~Hsiang Low, and Patrick Jaillet.
\newblock Incentives in private collaborative machine learning.
\newblock In \emph{Thirty-seventh Conference on Neural Information Processing Systems}, 2023.

\bibitem[Simonyan et~al.(2013)Simonyan, Vedaldi, and Zisserman]{simonyan2013deep}
Karen Simonyan, Andrea Vedaldi, and Andrew Zisserman.
\newblock Deep inside convolutional networks: Visualising image classification models and saliency maps.
\newblock \emph{arXiv preprint arXiv:1312.6034}, 2013.

\bibitem[Springenberg et~al.(2014)Springenberg, Dosovitskiy, Brox, and Riedmiller]{springenberg2014striving}
Jost~Tobias Springenberg, Alexey Dosovitskiy, Thomas Brox, and Martin Riedmiller.
\newblock Striving for simplicity: The all convolutional net.
\newblock \emph{arXiv preprint arXiv:1412.6806}, 2014.

\bibitem[Su et~al.(2023)Su, Tarr, and Muller]{su2023robust}
Peng Su, Garth Tarr, and Samuel Muller.
\newblock Robust variable selection under cellwise contamination.
\newblock \emph{Journal of Statistical Computation and Simulation}, pages 1--17, 2023.

\bibitem[Wachter et~al.(2017)Wachter, Mittelstadt, and Russell]{wachter2017counterfactual}
Sandra Wachter, Brent Mittelstadt, and Chris Russell.
\newblock Counterfactual explanations without opening the black box: Automated decisions and the gdpr.
\newblock \emph{Harv. JL \& Tech.}, 31:\penalty0 841, 2017.

\bibitem[Wang et~al.(2024{\natexlab{a}})Wang, Mittal, and Jia]{wang2024efficient}
Jiachen~T Wang, Prateek Mittal, and Ruoxi Jia.
\newblock Efficient data shapley for weighted nearest neighbor algorithms.
\newblock \emph{arXiv preprint arXiv:2401.11103}, 2024{\natexlab{a}}.

\bibitem[Wang et~al.(2024{\natexlab{b}})Wang, Yang, Zou, Kwon, and Jia]{wang2024rethinking}
Jiachen~T Wang, Tianji Yang, James Zou, Yongchan Kwon, and Ruoxi Jia.
\newblock Rethinking data shapley for data selection tasks: Misleads and merits.
\newblock \emph{arXiv preprint arXiv:2405.03875}, 2024{\natexlab{b}}.

\bibitem[Wang and Jia(2022)]{wang2022data}
Tianhao Wang and Ruoxi Jia.
\newblock Data banzhaf: A data valuation framework with maximal robustness to learning stochasticity.
\newblock \emph{arXiv preprint arXiv:2205.15466}, 2022.

\bibitem[Wu et~al.(2022)Wu, Shu, and Low]{wu2022davinz}
Zhaoxuan Wu, Yao Shu, and Bryan Kian~Hsiang Low.
\newblock Davinz: Data valuation using deep neural networks at initialization.
\newblock In \emph{International Conference on Machine Learning}, pages 24150--24176. PMLR, 2022.

\bibitem[Yoon et~al.(2020)Yoon, Arik, and Pfister]{yoon2020data}
Jinsung Yoon, Sercan Arik, and Tomas Pfister.
\newblock Data valuation using reinforcement learning.
\newblock In \emph{International Conference on Machine Learning}, pages 10842--10851. PMLR, 2020.

\bibitem[Zeiler and Fergus(2014)]{zeiler2014visualizing}
Matthew~D Zeiler and Rob Fergus.
\newblock Visualizing and understanding convolutional networks.
\newblock In \emph{Computer Vision--ECCV 2014: 13th European Conference, Zurich, Switzerland, September 6-12, 2014, Proceedings, Part I 13}, pages 818--833. Springer, 2014.

\bibitem[Zhang et~al.(2023)Zhang, Sun, Liu, Xiong, Pei, and Ren]{10.1145/3588728}
Jiayao Zhang, Qiheng Sun, Jinfei Liu, Li~Xiong, Jian Pei, and Kui Ren.
\newblock Efficient sampling approaches to shapley value approximation.
\newblock \emph{Proc. ACM Manag. Data}, 1\penalty0 (1), may 2023.
\newblock \doi{10.1145/3588728}.
\newblock URL \url{https://doi.org/10.1145/3588728}.

\bibitem[Zhao et~al.(2023)Zhao, Lyu, Fernandez, and Kolar]{zhao2023addressing}
Boxin Zhao, Boxiang Lyu, Raul~Castro Fernandez, and Mladen Kolar.
\newblock Addressing budget allocation and revenue allocation in data market environments using an adaptive sampling algorithm.
\newblock \emph{arXiv preprint arXiv:2306.02543}, 2023.

\end{thebibliography}








\newpage
\appendix
\normalsize

\section*{Supplementary Materials}
In the supplementary materials, we provide implementation details, additional experimental results, rigorous formalized proofs and data valuation experiment results. Code repository can be found at \url{https://github.com/yifansun99/2D-OOB-Joint-Valuation}.

\section{Implementation details}
\label{sec: implementation_details}
\subsection{Datasets}
\label{sec:datasets}
\paragraph{Tabular datasets} We use $12$ binary classification datasets obtained from OpenML \citep{feurer2021openml}. A summary of all the datasets is provided in Table \ref{datasummery}. These datasets are used in Section \ref{sec:cell-level outlier detection},~\ref{sec:cell fixation},~\ref{sec:ablation}, and Appendix~\ref{sec:data val exp}.

For each dataset, we first employ a standard normalization procedure, where each feature is normalized to have zero mean and unit standard deviation. After preprocessing, we randomly partition a subset of the data into two non-overlapping sets: a training dataset and a test dataset, which consists of $1000$ and $3000$ samples respectively. The training dataset is used to obtain the joint (or marginal) valuation for each cell (or data point). The test dataset is exclusively used for the cell fixation (or point removal) experiment when evaluating the test accuracy. Note that for methods that need a validation dataset such as \texttt{KNNShapley} and \texttt{DataShapley}, we additionally sample a separate validation dataset (disjoint from training dataset and test dataset) to evaluate the utility function. The size of the validation dataset is set to $10\%$ of the training sample size.

\paragraph{Image datasets} We create datasets by pairing CIFAR-10 classes, each pair consisting of a target attack class and a source class. The training and test dataset comprises $1000$ and $2000$ samples, respectively. The size of the validation dataset is set to $10\%$ of the training sample size. To manage the computational challenges posed by the baseline method, we employ the super-pixel technique to transform the ($32$,$32$,$3$) image into a $256$-dimensional vector. Specifically, we first average the pixel values across the three channels for each pixel. Then, we partition these transformed images into equal-sized $2 \times 2$ grids. Average pooling is applied within each grid to reduce pixel values to a single cell value, which is then arranged into a flattened input vector. A cell is annotated as poisoned if at least $25\%$ of its corresponding grid area contains the trigger.

\subsection{Implementation details for different methods}
\label{sec:implementation details}
\paragraph{\texttt{2D-OOB}} \texttt{2D-OOB} involves fitting a \textit{subset} random forest model with $B=1000$ decision trees based on the package ``scikit-learn''. When constructing each decision tree, we fix the feature subset size ratio as $0.5$ (unless otherwise specified). For Section \ref{sec:backdoor} and Appendix \ref{sec:data val exp}, we simply adopt $T(y_i, \hat{f}(x_{i, S_b})) = \mathds{1}(y_i = \hat{f}(x_{i, S_b}))$. For Section \ref{sec:cell-level outlier detection} and \ref{sec:cell fixation}, we further calculate the normalized negative $L_2$ distance between covariates and the class-specific mean in the bootstrap dataset, denoted as $d_{norm}(x_i,y_i)$. Then we use $T(y_i, \hat{f}(x_{i, S_b})) = \mathds{1}(y_i = \hat{f}(x_{i, S_b})) + d_{norm}(x_i,y_i)$.

\paragraph{\texttt{2D-KNN}} \texttt{2D-KNN} employs KNN as a surrogate model to approximate \texttt{2D-Shapley}. We set the number of nearest neighbors as $10$ and the number of permutations as $1000$. The hyperparameters are selected based on convergence behavior, and the run time is measured until the values converge.

\subsection{Implementation details for cell-level outlier generation}
\label{sec:cell-level outlier generation}

Following \citet{du2022vos} and \citet{liu20232d}, we replace a given cell with an outlier value. Here, the outlier value is randomly generated from the two-sided ``tails'' of the Gaussian distribution fitted to the column's mean and standard deviation, where the probability of the two-sided tail area is set to be $1$\%. In total, $4\%$ ($20\% \times 20\%$) of the cells are replaced with the corresponding outlier values.

\subsection{Implementation details for backdoor trigger generation}
\label{sec:backdoor trigger generation}

Following the prior work \citep{gu2017badnets,Liu2018TrojaningAO}, we generate the BadNets square and the Trojan square trigger. For BadNets, we adopt the implementation in \citet{art2018}. For Trojan Attack, we use a pretrained ResNet-18 model on the CIFAR-10 dataset and employ the implementation in \citet{pang:2022:eurosp}. For each attack, we evaluate its effectiveness by training a decision tree model on the poisoned dataset. The accuracy on a clean test set remains nearly unchanged compared to a model trained on an uncontaminated training set, while the attack success rate on a held-out poisoned test set is guaranteed to exceed $75\%$.

\begin{table}[!htbp]
    \centering
    \small
    \caption{\textbf{A summary of all the datasets used in \ref{sec:cell-level outlier detection},~\ref{sec:cell fixation}, and Appendix ~\ref{sec:data val exp}.} These datasets have been commonly used in previous literature \citep {ghorbani2019data, kwon2021beta, kwon2023data}}
    \label{dataset}
    \resizebox{\textwidth}{!}{
        \begin{tabular}{l|cccc}
        \toprule
         Name & Total sample size & Input dimension & Majority class proportion & OpenML ID\\
         \midrule
lawschool & 20800 & 6 & 0.679& 43890 \\
electricity & 38474 & 6 & 0.500  & 44080 \\
fried & 40768 & 10 & 0.502 & 901 \\
2dplanes & 40768 & 10 & 0.501 & 727  \\
creditcard & 30000 & 23 & 0.779& 42477 \\
pol & 15000 & 48 & 0.664 & 722 \\
MiniBooNE & 72998 & 50 & 0.500 & 43974  \\
jannis & 57580 & 54 & 0.500 & 43977 \\
nomao & 34465 & 89 & 0.715 & 1486  \\
vehicle\_sensIT & 98528 & 100 &0.500 & 357  \\ 
gas\_drift & 5935 & 128 & 0.507 & 1476 \\
musk  & 6598 & 166 & 0.846 & 1116 \\
        \midrule
        \end{tabular}}
\label{datasummery}
\end{table}

\section{Additional experimental results}

\subsection{Additional results for cell-level outlier detection}
\label{sec: two-stage}

\paragraph{Additional results on multi-class classification datasets}
We conducted cell-level outlier detection experiments (as described in Section \ref{sec:cell-level outlier detection}) on three multi-class classification datasets from the UCI Machine Learning repository \citep{kelly_longjohn_nottingham}. As shown in the Table \ref{tab:multi-class}, \texttt{2D-OOB} displays superior detection performance and efficiency.

\begin{table}[ht]
\centering
\caption{\textbf{Cell-level outlier detection results on multi-class classification datasets.} The average and standard error of the AUC and run-time (in
seconds) based on $30$ independent experiments are denoted by  ``average $\pm$ standard error''.}
\small
\resizebox{0.8\textwidth}{!}{
\begin{tabular}{l|c|c|c|c}
\hline
\toprule
\multirow{2}{*}{Dataset}  & \multicolumn{2}{c}{\textbf{AUC} $\uparrow$} & \multicolumn{2}{|c}{\textbf{Run-time} $\downarrow$} \\
& \texttt{2D-OOB} (ours) & \texttt{2D-KNN} & \texttt{2D-OOB} (ours) & \texttt{2D-KNN} \\
\midrule
Covertype & \textbf{0.81$\pm$0.0156} & 0.63$\pm$0.0183 & \textbf{3.98$\pm$0.5774} & 962.34$\pm$1.3383 \\
Dry Bean & \textbf{0.88$\pm$0.0059} & 0.85$\pm$0.0192 & \textbf{3.31$\pm$0.4586} & 347.80$\pm$2.0212 \\
Wine Quality & \textbf{0.86$\pm$0.0178} & 0.57$\pm$0.0252 & \textbf{2.90$\pm$0.1240} & 269.14$\pm$1.1825 \\
\hline
\end{tabular}}
\label{tab:multi-class}
\end{table}

\paragraph{Additional baseline: two-stage attribution}
Once we obtain the data valuation scores, an alternative solution approach to determining cell-level attributions involves leveraging feature attribution methods such as SHAP \citep{lundberg2017unified}. We explore an additional baseline method building upon this idea: initially, \texttt{Data-OOB} (or any other data valuation method) is computed for the $i$-th data point, denoted as $dv_i$. Subsequently, \texttt{TreeSHAP} \citep{lundberg2018consistent} is fitted, using $dv_i$ as the target and the concatenation of $x_i$ and $y_i$ (denoted as $x_i \oplus y_i$) as the predictor. The derived local feature attributions are then interpreted as joint valuation results. We refer to this method as ``two-stage attribution''.

Table \ref{tab:baseline_two_stage} indicates that \texttt{2D-OOB} substantially outperforms its two-stage counterpart. We hypothesize that the superiority of our method stems from integrating data valuation and feature attribution into a cohesive framework. Conversely, the two-stage method treats data valuation and feature attribution as separate processes, potentially resulting in sub-optimal outcomes. Furthermore, due to the computational complexity of \texttt{TreeSHAP}, the two-stage approach is notably slower compared to our method.

\begin{table*}[t]
\caption{\textbf{Cell-level outlier detection results (AUC) of \texttt{2D-OOB} and the two-stage attribution. }Our method shows a better performance than the alternative method by a significant performance margin.}
\label{tab:baseline_two_stage}
\vskip 0.15in
\begin{center}
\resizebox{0.5\textwidth}{!}{
\begin{tabular}{l|c|c}
\toprule
\multirow{2}{*}{Dataset}  & \multicolumn{2}{c}{\textbf{AUC} $\uparrow$} \\
& \texttt{2D-OOB} (ours) & Two-stage attribution \\
\midrule
lawschool        & \textbf{0.88$\pm$ 0.0027} & 0.83$\pm$ 0.0064 \\
electricity      & \textbf{0.77$\pm$ 0.0072} & 0.64$\pm$ 0.0093 \\
fried            & \textbf{0.91$\pm$ 0.0015} & 0.82$\pm$ 0.0068 \\
2dplanes         & \textbf{0.87$\pm$ 0.0015} & 0.80$\pm$ 0.0058 \\
creditcard       & \textbf{0.72$\pm$ 0.0028} & 0.67$\pm$ 0.0051 \\
pol              & \textbf{0.82$\pm$ 0.0014} & 0.78$\pm$ 0.0042 \\
MiniBooNE        & \textbf{0.77$\pm$ 0.0058} & 0.70$\pm$ 0.0041 \\
jannis           & \textbf{0.83$\pm$ 0.0042} & 0.62$\pm$ 0.0043 \\
nomao            & \textbf{0.79$\pm$ 0.0021} & 0.71$\pm$ 0.0041 \\
vehicle\_sensIT  & \textbf{0.81$\pm$ 0.0014} & 0.64$\pm$ 0.0033 \\
gas\_drift       & \textbf{0.86$\pm$ 0.0010} & 0.73$\pm$ 0.0143 \\
musk             & \textbf{0.88$\pm$ 0.0008} & 0.68$\pm$ 0.0028 \\
\midrule
Average          & \textbf{0.83}  & 0.72 \\
\bottomrule
\end{tabular}
}
\end{center}
\vskip -0.1in
\end{table*}

\paragraph{A noisy setting with more outlier cells} We consider a more challenging scenario with increased outlier levels, where both the row outlier ratio and column outlier ratio increase from $20\%$ (as in Section \ref{sec:cell-level outlier detection}) to $50\%$.  Consequently, this leads to $25\%$ ($50\% \times 50\%$) of the cells being replaced with outlier values.  We follow the same outlier generation procedure outlined in Appendix \ref{sec:cell-level outlier generation}. The findings, presented in Table \ref{tab:ablation_outlier_ratio}, demonstrate that our method maintains a significantly superior performance over \texttt{2D-KNN}, even under such a noisy setting.

\begin{table}[t]
\small
\caption{\textbf{Cell-level outlier detection results (AUC) of different joint valuation methods when the row outlier ratio and column outlier ratio are both $50\%$. }Our method consistently outperforms \texttt{2D-KNN} even in the presence of significant noise.}
\label{tab:ablation_outlier_ratio}
\vskip 0.15in
\begin{center}
\resizebox{0.5\textwidth}{!}{
\begin{tabular}{l|c|c}
\toprule
\multirow{2}{*}{Dataset}  & \multicolumn{2}{c}{\textbf{AUC} $\uparrow$} \\
& \texttt{2D-OOB} (ours) & \texttt{2D-KNN} \\
\midrule
lawschool        & \textbf{0.75$\pm$ 0.0084} & 0.60$\pm$ 0.0144 \\
electricity      & \textbf{0.64$\pm$ 0.0155} & 0.60$\pm$ 0.0106 \\
fried            & \textbf{0.74$\pm$ 0.0087} & 0.54$\pm$ 0.0027 \\
2dplanes         & \textbf{0.74$\pm$ 0.0063} & 0.55$\pm$ 0.0033 \\
creditcard       & \textbf{0.63$\pm$ 0.0055} & 0.61$\pm$ 0.0053 \\
pol              & \textbf{0.69$\pm$ 0.0069} & 0.60$\pm$ 0.0042 \\
MiniBooNE        & \textbf{0.67$\pm$ 0.0128} & 0.60$\pm$ 0.0048 \\
jannis           & \textbf{0.70$\pm$ 0.0113} & 0.53$\pm$ 0.0014 \\
nomao            & \textbf{0.70$\pm$ 0.0088} & 0.58$\pm$ 0.0052 \\
vehicle\_sensIT  & \textbf{0.70$\pm$ 0.0075} & 0.55$\pm$ 0.0031 \\
gas\_drift       & \textbf{0.73$\pm$ 0.0077} & 0.65$\pm$ 0.0114 \\
musk             & \textbf{0.77$\pm$ 0.0063} & 0.64$\pm$ 0.0038 \\
\midrule
Average          & \textbf{0.71}  & 0.59 \\
\bottomrule
\end{tabular}
}
\end{center}
\vskip -0.1in
\end{table}

\subsection{Additional results for cell fixation experiment}
\label{sec:additional fixation}

Figure \ref{Fig:cell_fixation_appendix} presents the results for the cell fixation experiment on $6$ additional datasets. \texttt{2D-OOB} excels in precisely detecting and correcting relevant cell outliers.

\begin{figure}[t]
  \centering
  \includegraphics[width = \textwidth]{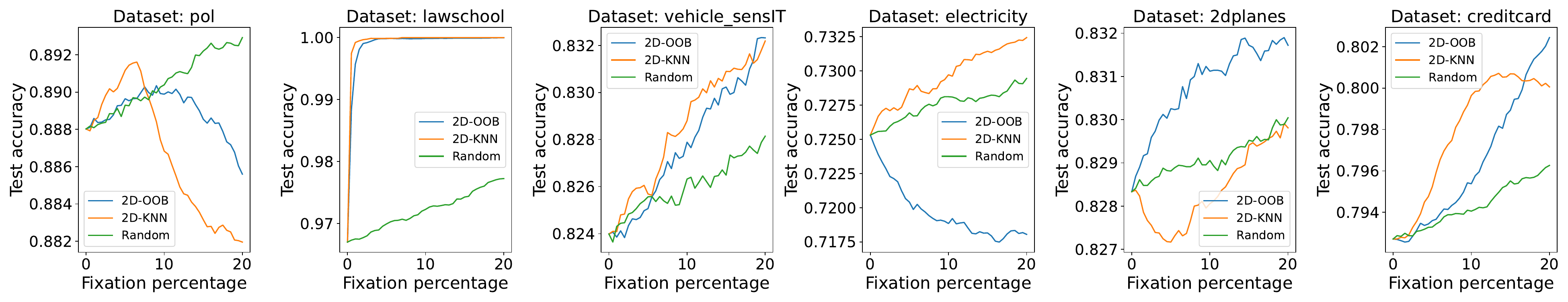}
  \vspace{-6mm}
  \caption{\textbf{Cell fixation experiment results (test accuracy curves) for \texttt{2D-OOB}, \texttt{2D-KNN} and a random baseline.} We replace cell values with ground-truth values from the cells with the lowest valuation to the highest valuation. The results from $6$ additional datasets are displayed. We conduct $30$ independent trials and report the average results. A higher curve indicates better performance. \texttt{2D-OOB} sets itself apart by its remarkable precision in detecting and rectifying relevant cell outliers.
  }
  \label{Fig:cell_fixation_appendix}
\end{figure}

\paragraph{The scenario without ground-truth knowledge} Following \citet{liu20232d}, we examine a situation where external information on the ground-truth annotations of outlier cells is not accessible. In this scenario, we address these outliers by substituting them with the average of other cells in the same feature column. This procedure starts by addressing cells with the lowest valuations, based on the hypothesis that correcting these cells is likely to maintain or potentially improve the model's performance. As depicted in Figure \ref{Fig:cell_removal}, \texttt{2D-OOB} conforms to this expected trend, demonstrating the effectiveness of our method in joint valuation. Conversely, \texttt{2D-KNN} fails to show similar performance improvements.

\begin{figure}[t]
  \centering
  \includegraphics[width = \textwidth]{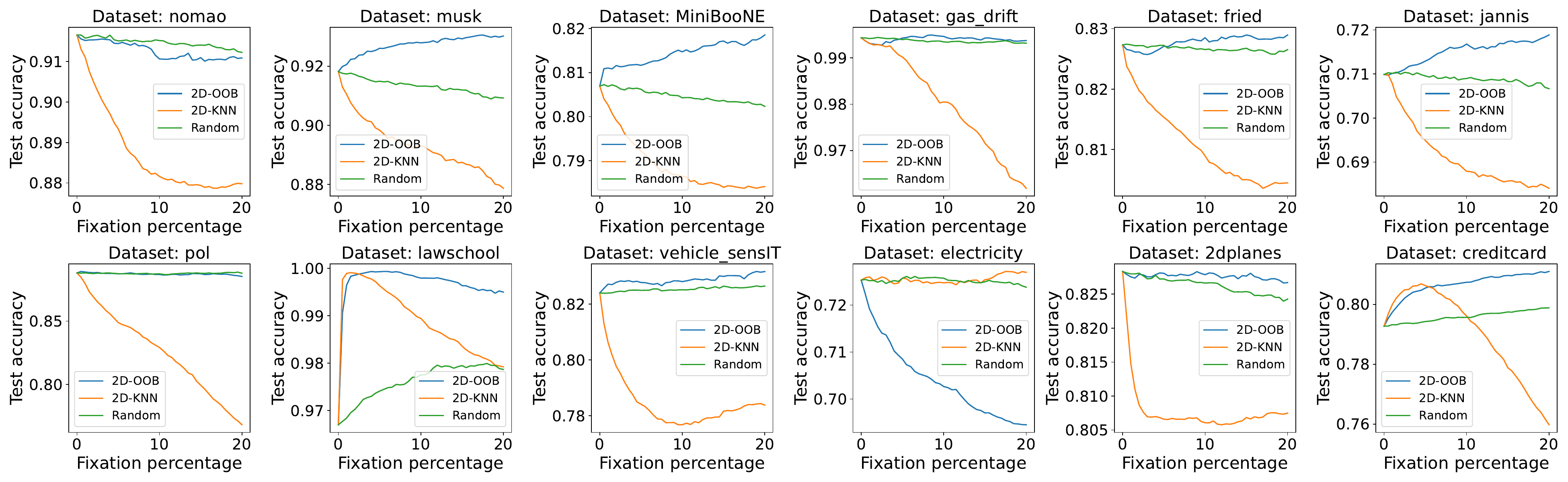}
  \vspace{-6mm}
  \caption{\textbf{Cell fixation experiment (without ground-truth knowledge) results (test accuracy curves) for \texttt{2D-OOB}, \texttt{2D-KNN} and a random baseline.} We replace cell values with column mean imputations from cells with the lowest value to the highest value. We report the average results of $30$ independent trials from $12$ datasets. A higher curve indicates better performance. 
  }
  \label{Fig:cell_removal}
\end{figure}

\subsection{Additional results for backdoor trigger experiment}
\label{sec:additional backdoor}
We provide additional qualitative examples of the backdoor trigger detection experiment in Figure \ref{Fig:backdoor appendix}.

\begin{figure}[h]
  \centering
  \includegraphics[width = \textwidth]{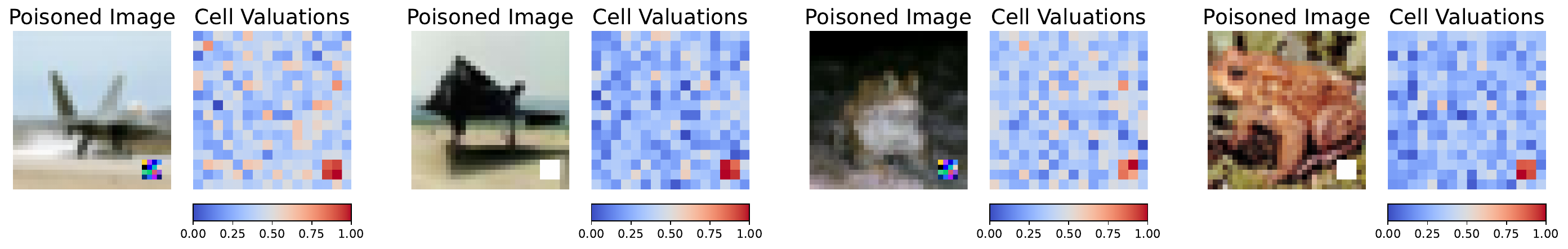}
  \vspace{-6mm}
  \caption{
    \textbf{Qualitative results on more datasets for the backdoor trigger detection experiment.} The first two images are from the class ``airplane'' but have been relabeled as ``automobile'', while the latter two images are from the class ``frog'' and have been relabeled as ``horse''. 
  }
  \label{Fig:backdoor appendix}
\end{figure}

\section{Proof of Proposition 3.1}
\label{sec:proof}

\begin{proof} 
For simplicity, we denote $\phi^{\mathrm{OOB}}_i(S)$ as $\phi_i(S)$ and $\psi^{\mathrm{2D-OOB}}_{ij}$ as $\psi_{ij}$ in the proof. Let $\mathcal{S} := \{S \subseteq [d]\}$ represent the set of all feature subsets, where $S$ is a feature subset. We denote the cardinality of $\mathcal{S}$ as $L:=|\mathcal{S}| = 2^d$. Let $\boldsymbol{\gamma}_{b}$ be a weight vector $\boldsymbol{\gamma}_{b} := (\gamma_{b1},\ldots,\gamma_{bL})$ for all $b\in[B]$, where $\gamma_{bl} \in \{0,1\}$ and $\gamma_{bl} = 1$ indicates the $l$-th subset is used in the $b$-th weak learner. With $\{\mathbf{w}_{b},\boldsymbol{\gamma}_{b},\hat{f}_b\}_{b=1}^B$, we can denote the $i$-th \texttt{Data-OOB} on the $l$-th feature subset $S_l$ as
\begin{align*}
   \phi_i(S_l) = \frac{\sum_{b=1}^B\mathds{1}(w_{bi}=0)\mathds{1}(\gamma_{bl}=1)T(y_i,\hat{f}_{b}(x_{i, S_l}))}{\sum_{b=1}^B\mathds{1}(w_{bi}=0)\mathds{1}(\gamma_{bl}=1)}.
\end{align*}

With slight abuse of notation, the formulation of \texttt{2D-OOB} in~\eqref{eqn:df_oob} can be expressed as follows.
\begin{align*}
    \psi_{ij}  &= \frac{\sum_{l=1}^{L}\sum_{b=1}^B\mathds{1}(w_{bi}=0)\mathds{1}(\gamma_{bl}=1)\mathds{1}(j\in S_l)T(y_i,\hat{f}_{b}(x_{i, S_l}))}{\sum_{l=1}^{L}\sum_{b=1}^B\mathds{1}(w_{bi}=0)\mathds{1}(\gamma_{bl}=1)\mathds{1}(j\in S_l)} \\
    &= \sum_{l=1}^{L}\mathds{1}(j\in S_l)\frac{\sum_{b=1}^B\mathds{1}(w_{bi}=0)\mathds{1}(\gamma_{bl}=1)T(y_i,\hat{f}_{b}(x_{i, S_l}))}{\sum_{l=1}^{L}\sum_{b=1}^B\mathds{1}(w_{bi}=0)\mathds{1}(\gamma_{bl}=1)\mathds{1}(j\in S_l)} \\
    &= \sum_{l=1}^{L}\mathds{1}(j\in S_l)\frac{\sum_{b=1}^B\mathds{1}(w_{bi}=0)\mathds{1}(\gamma_{bl}=1)}{\sum_{l=1}^{L}\sum_{b=1}^B\mathds{1}(w_{bi}=0)\mathds{1}(\gamma_{bl}=1)\mathds{1}(j\in S_l)}\frac{\sum_{b=1}^B\mathds{1}(w_{bi}=0)\mathds{1}(\gamma_{bl}=1)T(y_i,\hat{f}_{b}(x_{i, S_l}))}{\sum_{b=1}^B\mathds{1}(w_{bi}=0)\mathds{1}(\gamma_{bl}=1)} \\
    &= \sum_{l=1}^{L}\alpha_{i,j,l}\phi_i(S_l),\\
\end{align*}
where $\alpha_{i,j,l} \propto \mathds{1}(j\in S_l)\sum_{b=1}^B\mathds{1}(w_{bi}=0)\mathds{1}(\gamma_{bl}=1), \forall i \in [n], j \in [d], l \in [L]$ and $\sum_{l=1}^{L} \alpha_{i,j,l} = 1$. Define $P_i(S_l|j \in S_l, \{w_{bi}\}_{b=1}^{B}) = \alpha_{i, j,l}$, which specifies an empirical distribution of the feature subset $S$, conditioned on $j \in S$ and the bootstrap sampling process. Here, $\mathds{1} (j \in S_l)$ indicates that the distribution is conditioned on the inclusion of the $j$-th feature within the feature subset $S_l$.  $w _{bi}$ indicates whether the $i$-th sample is out-of-bag in the $b$-th bootstrap, and $\gamma _{bl}$ indicates whether the $l$-th feature subset is selected in the $b$-th weak learner. Thus, the point mass is determined by the sampling process, leading to:  \begin{align*}
    \psi_{ij}  = \mathbb{E}_{\hat{F}_S} [ \phi_i (S) \mid j \in S].
\end{align*}

\end{proof}

\section{Data valuation experiment}

\label{sec:data val exp}
In this section, we show that \texttt{2D-OOB-data}, the marginalization of \texttt{2D-OOB},  offers an effective approach to data valuation. This serves as the basis of our enhanced performance in joint valuation.

\paragraph{Marginalization}
\texttt{2D-OOB} aims to attribute data contribution through cells. Consequently, by summing up \texttt{2D-OOB} over all columns, we can derive data contribution values. For $i\in [n]$, we define the \texttt{2D-OOB-data} $\psi^{data}_{i}$ as follows.
\begin{equation}
    \psi^{data}_{i} := \frac{1}{d} \sum_{j=1} ^d \psi^{\mathrm{2D-OOB}}_{ij}.
    \label{eqn:2d_oob_data}
\end{equation}

Based on discussions in Section \ref{sec:theory}, the marginalizations also connect with \texttt{Data-OOB}:

\begin{proposition}
For all $i\in[n]$, the marginalizations $\psi^{data}_{i}$ can be expressed as follows.
\begin{align*}
    \psi^{data}_{i} &= \mathbb{E}_{\hat{F}_S} [ \phi^{OOB}_i (S)], 
\end{align*}
where the notations follow the same definitions as Proposition \ref{prop:representation}.
\label{prop:representation: marginalization}
\end{proposition}

\begin{proof}
    
Based on the definition of \texttt{2D-OOB-Data}, for $i \in [n]$, 
\begin{align*}
    \psi^{data}_{i} := \frac{1}{d} \sum_{j=1} ^d \psi^{\mathrm{2D-OOB}}_{ij} &=  \frac{1}{d} \sum_{j=1} ^d \sum_{l=1}^{L}\alpha_{i,j,l}\phi^{\mathrm{OOB}}_i(S_l) \\
    &= \sum_{l=1}^{L} (\frac{1}{d} \sum_{j=1}^d\alpha_{i,j,l}) \phi^{\mathrm{OOB}}_i(S_l),\\
\end{align*}
where $\alpha_{i,j,l}$ is defined in Appendix \ref{sec:proof}. We have $\sum_{l=1}^{L}(\frac{1}{d} \sum_{j=1}^d\alpha_{i,j,l}) = \frac{1}{d}\sum_{j=1}^d\sum_{l=1}^{L}\alpha_{i,j,l} = 1$. Denote $\mathbb{P}_i(S_l|\{w_{bi}\}_{b=1}^{B}) = \frac{1}{d} \sum_{j=1}^d\alpha_{i,j,l}$, which induces the empirical expectation of \texttt{Data-OOB} with respect to $S_l$. 
\end{proof}

Proposition \ref{prop:representation: marginalization} indicates \texttt{2D-OOB-data} $\psi^{data}_{i}$ can be expressed as the average \texttt{Data-OOB} value for the $i$-th data point. As a result, \texttt{2D-OOB-data} is expected to inherit the advanced ability of \texttt{Data-OOB} in terms of data valuation, as will be empirically examined next.

\paragraph{Experimental setting} Following the standard protocol in \citet{kwon2021beta,kwon2023data} and \citet{jiang2023opendataval}, we randomly select $10$\% of the data points and flip its label to the other class. For joint valuation methods, we calculate the valuation of each cell and perform the marginalization over features to obtain the data valuation scores. Mislabeled data detection and data removal experiments are examined based on this setting. For the baseline methods, we further incorporate several state-of-the-art data valuation methods including \texttt{DataShapley} \citep{ghorbani2019data}, \texttt{KNNShapley} \citep{jia2019efficient}, \texttt{DataBanzhaf} \citep{wang2022data}, \texttt{LAVA} \citep{just2023lava}, and \texttt{Data-OOB} \citep{kwon2023data}. Implementation details are listed below. To guarantee a fair comparison, we also employ the decision tree as the base model in \texttt{DataShapley} and \texttt{DataBanzhaf}.  We adopt the same $12$ datasets as outlined in Section \ref{sec:cell-level outlier detection}.

\paragraph{\texttt{Data-OOB}} \texttt{Data-OOB} involves fitting a random forest model without feature subset sampling, consisting of $1000$ decision trees.

\paragraph{\texttt{DataShapley}} We use a Monte Carlo-based algorithm. The Gelman-Rubin statistics is computed to determine the termination criteria of the algorithm. Following \citet{jiang2023opendataval}, We adopt the threshold to be $1.05$. 

\paragraph{\texttt{KNNShapley}} We set the number of nearest neighbors to be $10\%$ of the sample size following \citet{jia2019efficient}.

\paragraph{\texttt{LAVA}} We calculate the class-wise Wasserstein distance following \citet{just2023lava}. The ``OTDD'' framework is adopted to complete the optimal transport calculation.

\paragraph{\texttt{DataBanzhaf}} We adopt the implementation from \citet{jiang2023opendataval}. We set ``the number of models to train'' as $1000$.

\subsection{Mislabeled data detection} We calculate the precision-recall curve by comparing the actual annotations, which denote whether data points are mislabeled, against the data valuation scores computed by different methods. Mislabeled data typically have a detrimental impact on model performance. Therefore, data points that receive a lower valuation score are considered to have a higher chance of being mislabeled. We then determine AUCPR (the AUC of the precision-recall curve) as a quantitative metric to assess the detection efficacy.

As shown in Table \ref{tab:mislabeled}, \texttt{2D-OOB-data} consistently outperforms \texttt{2D-KNN-data} across all datasets, suggesting its superior ability to detect mislabeled data points. It is worth noting that \texttt{2D-OOB-data}'s results are on par with \texttt{Data-OOB}, while significantly exceeding the performance of other data valuation methods. These results are in line with our theoretical analysis regarding the resemblance between \texttt{Data-OOB} and \texttt{2D-OOB-data}. However, it is important to highlight that applying \texttt{Data-OOB} to the joint tasks is not feasible as mentioned earlier, underscoring the necessity for the development of \texttt{2D-OOB}.

\begin{table*}[t]
\caption{\textbf{Point-level mislabeled data detection results.} AUCPR of different data valuation and (marginalized) joint valuation methods. The average and standard error of the AUCPR based on $30$ independent experiments are denoted by ``average $\pm$ standard error''. Bold numbers denote the best method, for data valuation and joint valuation respectively. The AUCPR value for the \texttt{Random} method consistently remains at 0.5 across all datasets. \texttt{2D-OOB-data} exhibits performance comparable to \texttt{Data-OOB}, while significantly surpassing \texttt{2D-KNN-data} (the marginalization of \texttt{2D-KNN}) and all other data valuation methods.
}
\label{tab:mislabeled}
\begin{center}
\resizebox{\textwidth}{!}{
\begin{tabular}{l|ccccc|cc}
\toprule
\multirow{2}{*}{Dataset}  & \multicolumn{5}{c|}{\textbf{Data Valuation}} & \multicolumn{2}{c}{\textbf{Joint Valuation (Marginalized)}} \\
& \texttt{KNNShapley} & \texttt{LAVA} & \texttt{DataBanzhaf} & \texttt{DataShapley} & \texttt{Data-OOB} & \texttt{2D-KNN-data} & \texttt{2D-OOB-data} (ours) \\
\midrule
lawschool           & 0.66 $\pm$ 0.013 & 0.13 $\pm$ 0.003 & 0.46 $\pm$ 0.008 & 0.88 $\pm$ 0.007 & \textbf{1.00 $\pm$ 0.000} & 0.46 $\pm$ 0.011 & \textbf{0.99 $\pm$ 0.002} \\
electricity         & 0.22 $\pm$ 0.008 & 0.11 $\pm$ 0.002 & 0.18 $\pm$ 0.005 & 0.26 $\pm$ 0.007 & \textbf{0.44 $\pm$ 0.007} & 0.20 $\pm$ 0.006 & \textbf{0.39 $\pm$ 0.007} \\
fried              & 0.40 $\pm$ 0.014 & 0.11 $\pm$ 0.002 & 0.22 $\pm$ 0.007 & 0.35 $\pm$ 0.009 & \textbf{0.76 $\pm$ 0.007} & 0.34 $\pm$ 0.010 & \textbf{0.73 $\pm$ 0.008} \\
2dplanes           & 0.46 $\pm$ 0.016 & 0.12 $\pm$ 0.002 & 0.32 $\pm$ 0.007 & 0.54 $\pm$ 0.009 & \textbf{0.78 $\pm$ 0.008} & 0.44 $\pm$ 0.011 & \textbf{0.68 $\pm$ 0.010} \\
creditcard         & 0.37 $\pm$ 0.007 & 0.11 $\pm$ 0.003 & 0.16 $\pm$ 0.004 & 0.28 $\pm$ 0.006 & \textbf{0.40 $\pm$ 0.007} & 0.20 $\pm$ 0.005 & \textbf{0.40 $\pm$ 0.007} \\
pol                & 0.19 $\pm$ 0.017 & 0.11 $\pm$ 0.002 & 0.37 $\pm$ 0.010 & 0.58 $\pm$ 0.012 & \textbf{0.93 $\pm$ 0.004} & 0.29 $\pm$ 0.018 & \textbf{0.87 $\pm$ 0.005} \\
MiniBooNE          & 0.41 $\pm$ 0.013 & 0.13 $\pm$ 0.006 & 0.23 $\pm$ 0.007 & 0.41 $\pm$ 0.010 & \textbf{0.78 $\pm$ 0.007} & 0.36 $\pm$ 0.008 & \textbf{0.78 $\pm$ 0.007} \\
jannis             & 0.20 $\pm$ 0.007 & 0.11 $\pm$ 0.002 & 0.14 $\pm$ 0.003 & 0.17 $\pm$ 0.005 & \textbf{0.38 $\pm$ 0.010} & 0.19 $\pm$ 0.006 & \textbf{0.37 $\pm$ 0.010} \\
nomao              & 0.61 $\pm$ 0.012 & 0.14 $\pm$ 0.003 & 0.33 $\pm$ 0.010 & 0.58 $\pm$ 0.009 & \textbf{0.87 $\pm$ 0.006} & 0.33 $\pm$ 0.011 & \textbf{0.88 $\pm$ 0.005} \\
vehicle\_sensIT   & 0.22 $\pm$ 0.009 & 0.11 $\pm$ 0.002 & 0.21 $\pm$ 0.007 & 0.33 $\pm$ 0.011 & \textbf{0.56 $\pm$ 0.010} & 0.14 $\pm$ 0.005 & \textbf{0.56 $\pm$ 0.010} \\
gas\_drift        & 0.87 $\pm$ 0.013 & 0.16 $\pm$ 0.006 & 0.42 $\pm$ 0.009 & 0.75 $\pm$ 0.008 & \textbf{0.98 $\pm$ 0.002} & 0.88 $\pm$ 0.006 & \textbf{0.98 $\pm$ 0.002} \\
musk              & 0.33 $\pm$ 0.010 & 0.11 $\pm$ 0.003 & 0.31 $\pm$ 0.007 & 0.47 $\pm$ 0.012 & \textbf{0.85 $\pm$ 0.005} & 0.21 $\pm$ 0.008 & \textbf{0.85 $\pm$ 0.005} \\
\midrule
Average             & 0.41 & 0.12 & 0.28 & 0.47 & \textbf{0.73} & 0.34 & \textbf{0.71} \\
\bottomrule
\end{tabular}
}
\end{center}
\vskip -0.1in
\end{table*}

\subsection{Point removal experiment} Removing low-quality data points has the potential to enhance model performance. Based on this idea, we employ the point removal experiment, a widely used benchmark in data valuation \citep{kwon2023data,ghorbani2019data,kwon2021beta}. According to the calculated data valuation scores, we progressively remove data points from the dataset in \textit{ascending} order. Specifically, we begin by removing the data points with the lowest data valuations. Each time we remove a datum, we fit a logistic model and use the held-out test set consisting of $3000$ instances to evaluate the model performance. The expected behavior is that the model performance will improve initially as the detrimental data points are gradually eliminated from the training process. Removing an excessive number of data points may result in a drastically altered dataset.
Consequently, we opt to remove the bottom $20\%$ data points.

Test accuracy curves throughout the data removal process are shown for $12$ datasets (Figure \ref{Fig:point removal}). A higher curve signifies better performance in terms of data valuation. Overall, \texttt{2D-OOB-data} demonstrates similar performance to \texttt{Data-OOB}, while significantly outperforming all other data valuation methods and the random baseline. When a few data points with poor quality are removed, the test performance of \texttt{2D-OOB-data} exhibits an evident increase. However, such a positive trend does not apply to other popular data valuation methods including \texttt{DataShapley} and \texttt{LAVA}. These findings highlight the potential of \texttt{2D-OOB-data} in selecting a subset of critical data points that can maintain model performance when the dataset is pruned.

\begin{figure}
  \centering
  \includegraphics[width = 0.9\textwidth]{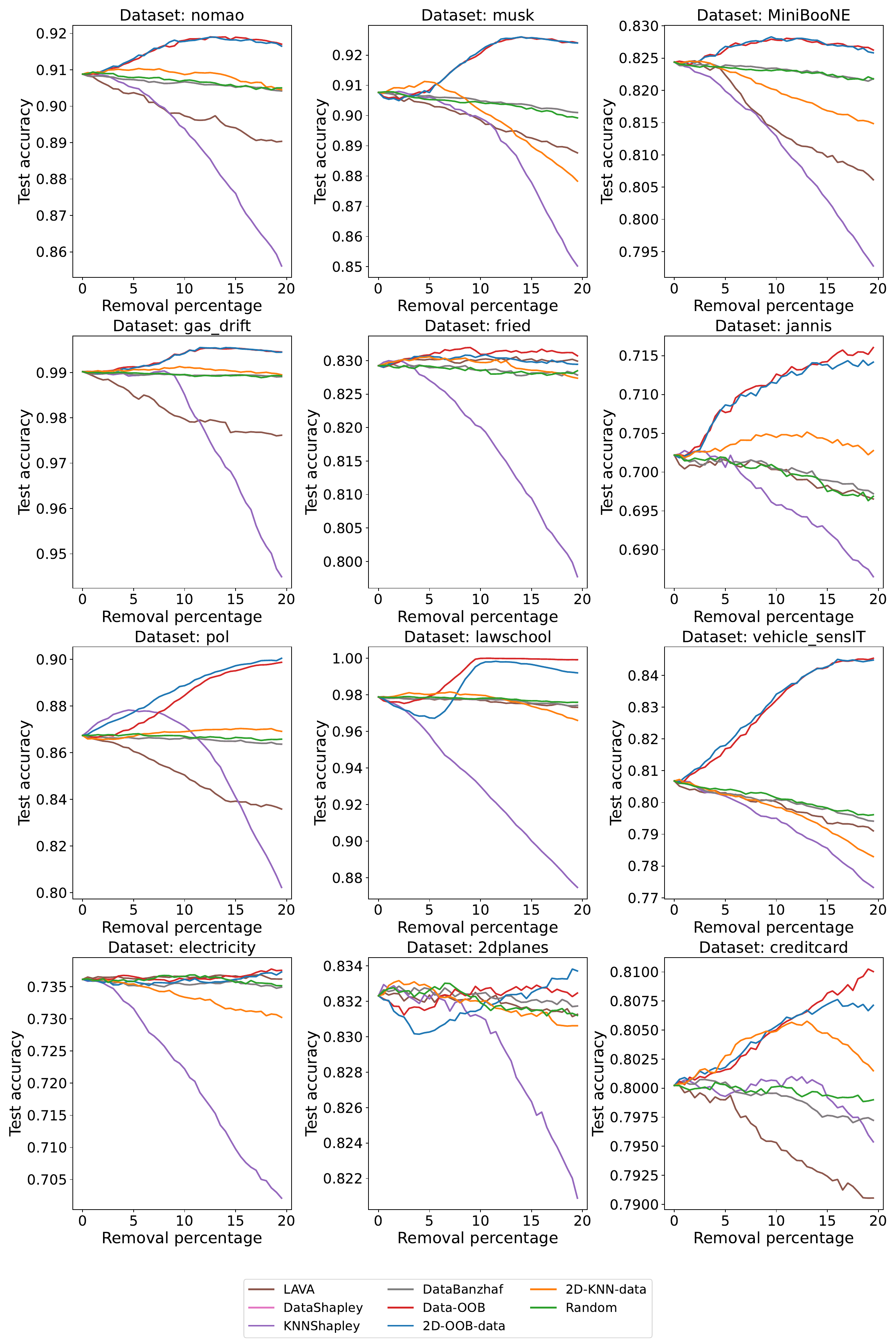}
  \vspace{-1mm}
  \caption{\textbf{Point removal experiment results (test accuracy curves) of $7$ data valuation methods -- \texttt{2D-OOB-data}, \texttt{2D-KNN-data}, \texttt{Data-OOB}, \texttt{LAVA}, \texttt{DataBanzhaf}, \texttt{DataShapley}, \texttt{KNNShapley} and a random baseline.} We remove data points from the lowest valuation to the highest valuation. The results from $6$ binary classification datasets are displayed. For each dataset, we conduct $30$ independent trials and report the average results.
  A higher curve indicates better performance. \texttt{2D-OOB-data} demonstrates superior ability in finding a set of helpful data points.   }
  \label{Fig:point removal}
\end{figure}

\end{document}